\documentclass[11pt]{article}

\usepackage{fullpage}
\usepackage{amsmath}
\usepackage{amsfonts}
\usepackage{amsthm}
\usepackage{smile}
\usepackage{subfigure}

\usepackage{tablefootnote}

\usepackage[colorlinks,
            linkcolor=red,
            anchorcolor=blue,
            citecolor=blue
            ]{hyperref}
\usepackage{enumitem}
\usepackage{mathtools}

\usepackage{colortbl}
\definecolor{LightCyan}{rgb}{0.8, 0.9, 1}

\usepackage{amssymb}
\usepackage{pifont}

\ifdefined\final
\usepackage[disable]{todonotes}
\else
\usepackage[textsize=tiny]{todonotes}
\fi
\allowdisplaybreaks

\title{\huge On the Optimal Sample Complexity of Offline Multi-\\Armed Bandits with KL Regularization}

\author{
    Kaixuan Ji\thanks{Equal contribution} \thanks{Department of Computer Science, University of California, Los Angeles, CA 90095, USA; e-mail: {\tt kaixuanji@cs.ucla.edu}} 
    ~~
    Qiwei Di\footnotemark[1] \thanks{Department of Computer Science, University of California, Los Angeles, CA 90095, USA; e-mail: {\tt qiwei2000@cs.ucla.edu}}
    ~~
    Heyang Zhao\thanks{Department of Computer Science, University of California, Los Angeles, CA 90095, USA; e-mail: {\tt hyzhao@cs.ucla.edu}} 
    ~~
    Qingyue Zhao \thanks{Department of Computer Science, University of California, Los Angeles, CA 90095, USA; e-mail: {\tt zhaoqy24@cs.ucla.edu}}
    ~~
    Quanquan Gu\thanks{Department of Computer Science, University of California, Los Angeles, CA 90095, USA; e-mail: {\tt qgu@cs.ucla.edu}}
}

\date{}

\newcommand{\piref}{\pi^{\mathsf{ref}}}
\newcommand{\kl}[2]{\ensuremath{{\mathsf{KL}}\left(#1\|#2\right)}}

\newcommand{\fls}{{\bar{r}}} 
\newcommand{\fps}{{\hat{r}}} 

\newcommand{\fgt}{{r^*}} 

\def \algcb {\text{KL-PCB}}

\newcommand{\pistar}{\pi^*}
\newcommand{\pihat}{\hat{\pi}}
\newcommand{\subopt}{\mathrm{SubOpt}}

\newcommand{\unif}{\mathsf{Unif}}

\newcommand{\rE}{\mathbb E}

\newcommand{\KL}{\mathsf{KL}}

\newcommand{\la}{\left\langle}
\newcommand{\ra}{\right\rangle}

\newcommand{\E}{\rE}

\begin{document}

\maketitle

\begin{abstract}
Kullback-Leibler (KL) regularization is widely used in offline decision-making and offers several benefits, motivating recent work on the sample complexity of offline learning with respect to \emph{KL-regularized performance metrics}. Nevertheless, the exact sample complexity of KL-regularized offline learning remains far from fully characterized. In this paper, we study this question in the setting of multi-armed bandits (MABs). We provide a sharp analysis of \algcb{}~\citep{zhao2026towards}, showing that with $n$ samples it achieves a suboptimality gap of $\tilde{O}(\eta SAC^{\pi^*}/n)$ under large regularization $\eta = \tilde{O}(\sqrt{n/(SAC^{\pi^*})})$, and $\tilde{O}(\sqrt{SAC^{\pi^*}/n})$ under small regularization $\eta = \tilde{\Omega}(\sqrt{n/(SAC^{\pi^*})})$, where $S$ is the number of contexts, $A$ is the number of arms, $C^{\pi^*}$ is the policy coverage coefficient at the optimal policy $\pi^*$, and $\tilde{O}$ and $\tilde{\Omega}$ hide all poly-logarithmic factors. We further provide a pair of sharper minimax lower bounds that match the upper bounds across the entire range of regularization strengths. Overall, our results provide a nearly complete characterization of offline multi-armed bandits with KL regularization.
\end{abstract}

\section{Introduction}

Offline reinforcement learning (RL) algorithms that learn from pre-collected data without interactive data collection have recently become appealing in both embodied settings~\citep{levine2018learning,levine2020offline} and language modeling~\citep{rafailov2023direct,ethayarajh2024model,meng2024simpo,rafailov2024from} due to their computational and memory efficiency, ease of implementation, and strong safety guarantees, especially when the interaction is risky. In this paradigm, algorithms typically employ divergence constraints that keep the learned policy close to a reference policy $\piref$~\citep{wu2019behavior,kumar2020conservative,rafailov2023direct}. Among these, the reverse Kullback--Leibler (KL) divergence regularizer $\kl{\hat\pi}{\piref}$ has become a popular and effective choice, especially for fine-tuning large language models~\citep{rafailov2023direct,ethayarajh2024model,meng2024simpo,lai2024step,rafailov2024from}.

The prevalence of KL regularization has motivated a line of statistical analysis of offline learning with respect to the \emph{KL-regularized performance metric}~\citep{zhao2025sharp,zhao2026towards,wu2025greedy}. \citet{zhao2025sharp} provide the first pair of $n^{-1}$-type upper and lower bounds of the suboptimality relative to the optimal \emph{KL-regularized} policy. However, the upper bound requires a strict uniform coverage condition, while the lower bound does not characterize the dependence on concentrability. \citet{zhao2026towards} take a step further by showing that \emph{single-policy concentrability} is necessary and sufficient for achieving the $n^{-1}$-type fast convergence rate. Nevertheless, there is a discrepancy in the notion of concentrability between the upper and lower bounds, which implies looseness in the worst case (e.g., in tabular settings). Moreover, both \citet{zhao2025sharp,zhao2026towards} achieve the $n^{-1}$-type rates only when the regularization is sufficiently strong. More recently, \citet{wu2025greedy} derive algorithms for this setting without pessimism, at the cost of an exponentially worse dependence on the inverse regularization intensity. Therefore, to the best of our knowledge, the following question remains open even for the simplest offline-learning models.
\begin{center}
    \emph{What is the sample complexity of offline decision making with KL regularization?}
\end{center}
In this paper, we provide a nearly comprehensive answer to this question for MABs\footnote{We consider a slightly generalized setting similar to the problem described in~\citet[Section~4]{rashidinejad2021bridging}, which allows the presence of different contexts.} (up to logarithmic factors). We dichotomize the problem into two regimes based on the strength of KL regularization. In each regime, we show that the sample complexity of $\algcb$\footnote{Despite a minor adaptation to the multi-armed setting, our algorithm largely follows the original $\algcb$ in~\citet{zhao2026towards}; thus, we still refer to our algorithm as $\algcb$.}~\citep{zhao2026towards} matches the corresponding statistical limit, thereby providing a near-complete characterization of offline learning for multi-armed bandits with KL regularization. Our main contributions are summarized as follows:
\begin{itemize}[leftmargin=*]
    \item We provide a sharp analysis of $\algcb$~\citep{zhao2026towards} in the multi-armed bandit setting. In particular, for $\algcb$, we provide a $\tilde{O}(\eta SAC^{\pi^*}/n)$ suboptimality upper bound under large regularization $\eta = \tilde{O}(\sqrt{n/(SAC^{\pi^*})})$, and a $\tilde{O}(\sqrt{SAC^{\pi^*}/n})$ suboptimality upper bound in the small regularization regime $\eta = \tilde{\Omega}(\sqrt{n/(SAC^{\pi^*})})$. Along with the sharp analysis, we further establish two lower bounds that match the upper bounds up to logarithmic factors in both regimes, indicating that $\algcb$ is near-optimal.
    

    \item We construct two distinct types of hard instances designed to exploit the unique structure of KL regularization under different regimes. Specifically, in the small regularization regime, we show that its statistical limit is closely related to the hardness of learning MABs with multiple optima~\citep{degenne2019pure,zhu2020regret,de2021bandits} in the offline setting.
    In the presence of multiple optima, {the direct application of the standard lower bound techniques: Le Cam's method (Assouad's method) and Fano's methods, is not sharp enough}. We address this with a novel pairing-and-counting technique which may be of independent interest beyond our setting.
    \item As a by-product, we identify the problem of learning offline MABs with multiple optima. Restricting to a uniform behavior policy, we provide an $\tilde{\Omega}(K/\sqrt{An})$ suboptimality lower bound, where $A$ is the number of arms and $K$ is the number of suboptimal arms. We further provide the suboptimality upper bound of a minimalist algorithm, empirical best-arm selection, for this setting, which certifies that the $\tilde{\Omega}(K/\sqrt{An})$ lower bound is tight up to logarithmic factors.
\end{itemize}
For ease of comparison, we summarize our results on KL-regularized MABs, together with representative results from related works, in Table~\ref{tab:kl-regularized}.

\newcolumntype{g}{>{\columncolor{LightCyan}}c}
\begin{table*}[t!]
\centering
\caption{Comparison of suboptimality upper and lower bounds for offline KL-regularized bandits. In this table, $n$ is the sample size, $\eta$ is the regularization parameter and $C^{\pi^*}$ is the (density ratio-based) single-policy coverage coefficient. For the multi-armed setting, $S$ is the size of context set and $A$ is the size of action set. For the function approximation setting, $D^2$ is the ($D^2$-based) all-policy concentrability, $D^2_{\pi^*}$ is the ($D^2$-based) single-policy concentrability at the optimal policy $\pi^*$ and $N_{\cR}$ is the covering number of the function class. When reducing to the multi-armed setting, $D^2 \geq D^2_{\pi^*}$ and $D^2_{\pi^*} = \Theta(SAC^{\pi^*})$ in the worst case (Appendix~\ref{app:coverage-d2}), $\log N_{\cR} = \tilde{\Theta}(SA)$ for the upper bound (Appendix~\ref{app:covering-number}) and should be interpreted as $\Theta(S)$ when it appears in the lower bounds (see Remark~\ref{rmk:kl-lower-comparison}). $\tilde{O}(\cdot)$ and $\tilde{\Omega}(\cdot)$ hide logarithmic factors.}
\vspace{1ex}
\renewcommand{\arraystretch}{1}
\resizebox{0.99\columnwidth}{!}{
{
\begin{tabular}{lgggg}
\toprule
\rowcolor{white} & & & Large Regularization & Small Regularization \\
\rowcolor{white} \multirow{-2}{*}{Type} & \multirow{-2}{*}{Algorithm}  & \multirow{-2}{*}{Setting} & $\eta^2 = \tilde{O}(n(SAC^{\pi^*})^{-1})$ & $\eta^2 = \tilde{\Omega}(n(SAC^{\pi^*})^{-1})$   \\
\midrule
\rowcolor{white}
& TMPS \\\rowcolor{white}
&\small\citep{zhao2025sharp}& \multirow{-2}{*}{Function Approximation} &     \multirow{-2}{*}{$\tilde{O}\Big(\frac{\eta D^2 \log N_{\cR}}{n}\Big)$}      &     \multirow{-2}{*}{N/A}     \\\rowcolor{white}
& KL-PCB \\ \rowcolor{white}
&\small\citep{zhao2026towards} & \multirow{-2}{*}{Function Approximation}&   \multirow{-2}{*}{$\tilde{O}\Big(\frac{\eta D^2_{\pi^*} \log N_{\cR}}{n}\Big)$}     &        \multirow{-2}{*}{N/A}   \\ \rowcolor{white}
\multirow{-1}{*}{Upper Bound}  & Greedy Sampling \\ \rowcolor{white}
&\small\citep{wu2025greedy} & \multirow{-2}{*}{Preference Feedback}&   \multirow{-2}{*}{$\tilde{O}\Big(\frac{\eta \exp(\eta) \log N_{\cR}}{n}\Big)$}     &        \multirow{-2}{*}{N/A}   \\
 & KL-PCB & & &\\
&\small(This Work) & \multirow{-2}{*}{Multi-armed} &     \multirow{-2}{*}{$\tilde{O}\Big(\frac{\eta SAC^{\pi^*} }{n}\Big)$}      & \multirow{-2}{*}{$\tilde{O}\Big(\sqrt{\frac{SAC^{\pi^*} }{n}}\Big)$}      \\
\midrule
\rowcolor{white} 
& \citet{zhao2025sharp} & Function Approximation & $\Omega\Big(\frac{\eta\log N_{\cR}}{n}\Big)$   &    N/A     \rule{0pt}{3ex}  \\[6pt] \rowcolor{white}
\multirow{-1}{*}{Lower Bound} & \citet{zhao2026towards} & Function Approximation & $\Omega\Big(\frac{\eta C^{\pi^*}\log N_{\cR}}{n}\Big)$   &    $\Omega\Big(\sqrt{\frac{C^{\pi^*}\log N_{\cR}}{n}}\Big)$     \rule{0pt}{2.5ex}  \\[6pt] 
& This Work & Multi-armed & $\tilde{\Omega}\Big(\frac{\eta SAC^{\pi^*}}{n}\Big)$  &    $\tilde{\Omega}\Big(\sqrt{\frac{SAC^{\pi^*}}{n}}\Big)$  \rule{0pt}{3.5ex}\\[6pt] 
\bottomrule
\end{tabular}
}}
\label{tab:kl-regularized}
\end{table*}

\paragraph{Notation.} The sets $\cS$ and $\cA$ are assumed to be finite throughout the paper.
For nonnegative sequences $\{x_n\}$ and $\{y_n\}$, we write $x_n = O(y_n)$ if $\limsup_{n\to\infty}{x_n}/{y_n} < \infty$, $x_n = o(y_n)$ if $\limsup_{n\to\infty}{x_n}/{y_n} =0$, $y_n = \Omega(x_n)$ (interchangeably written as $y_n \gtrsim x_n$) if $x_n = O(y_n)$, and $y_n = \Theta(x_n)$ if $x_n = O(y_n)$ and $x_n = \Omega(y_n)$. We further employ $\tilde{O}(\cdot), \tilde{\Omega}(\cdot)$, and $\tilde{\Theta}(\cdot)$ to hide $\polylog$ factors. 
For finite $\cX$ and $\cY$, we denote the family of probability kernels from $\cX$ to $\cY$ by $\cP(\cY|\cX)$. For a pair of probability measures $P \ll Q$ on the same space, we use $\kl{P}{Q} \coloneqq \int \log({\ud P}/{\ud Q})\ud P$ to denote their KL divergence. 
We denote by $\mathsf{Unif}(\cX)$ the uniform distribution on finite set $\cX$. For some $\pi \in \cP(\cX)$ with full support and $\cY \subseteq \cX$, we use $\pi|_{\cY} \in \cP(\cY)$ to denote the distribution such that $(\pi|_{\cY})(y) \propto \pi(y)$.
We denote $[N] \coloneqq \{1, \cdots, N\}$ for any positive integer $N$. Boldfaced lowercase letters are reserved for vectors. For $\xb, \yb \in \cX^n$, we use $d_H(\xb, \yb)$ to denote their Hamming distance. For $x, y \in \RR$, we denote $x \vee y = \max\{x,y\}$. We use $\mathsf{Bern}(p)$ to denote the Bernoulli distribution with mean $p$, and $\mathsf{Bin}(n,p)$ for the binomial distribution with $n$ trials and success rate $p$. For probability measure $P$, we use $\mathrm{supp}(P)$ to denote the support of $P$.

\section{Related Work}

\paragraph{Pessimism in Offline RL.} 
The principle of pessimism, under which the learner acts conservatively in the face of uncertainty, plays a crucial role in offline RL. In the tabular setting, pessimism is applied in both model-based~\citep{rashidinejad2022optimal,xie2021policy} and model-free~\citep{shi2022pessimistic,yan2023efficacy} algorithm design, based on which the optimal sample complexity is achieved~\citep{li2024settling}. Under reward or value function approximation, the pessimism principle is also widely adopted~\citep{jin2021pessimism,yin2022near,xie2021bellman,uehara2021pessimistic,wang2025modelbased} and achieves optimal sample complexity in both linear~\citep{xiong2023nearly} and general function approximation~\citep{di2024pessimistic}. Under KL-regularized objectives, pessimism is again provably beneficial and admits a sharp analysis~\citep{zhao2026towards}. 

\paragraph{KL-Regularized Bandit and RL.}  Several recent studies \citep{xie2025exploratory,xiong2024iterative,zhao2025sharp,foster2025good} investigated the statistical efficiency of KL-regularized objective, which provably enjoys an $\cO(n^{-1})$ suboptimality rate.
This fast rate was first demonstrated by \citet{tiapkin2023fast} in the setting of pure-exploration for maximum-entropy RL. For the setting of online regret minimization, \citet{zhao2025logarithmic} obtained a $\tilde{O}(\eta d_{\cR} \log N_{\cR} \log T)$ regret upper bound under function approximation, which was then improved to the near-optimal $\tilde{O}(\eta A\log^2 T \wedge \sqrt{AT})$ rate when specialized to multi-armed bandits~\citep{ji2026nearoptimal}.
In the pure offline setting, \citet{zhao2025sharp} established the optimal $\cO(n^{-1})$ suboptimality rate, albeit under the strong assumption that the behavior policy $\piref$ provides uniform coverage over the entire function class for all policies. 
This requirement was subsequently removed by~\citet{zhao2026towards} via pessimism, which yields an $\tilde{O}(\eta D^2_{\pi^*}\log N_{\cR}\, n^{-1})$ upper bound and an $\Omega(\eta C^{\pi^*}\log N_{\cR}\, n^{-1})$ lower bound. \citet{wu2025greedy} show that one can also achieve an $\tilde{O}(\eta \exp(\eta)\log N_{\cR}\, n^{-1})$ upper bound with greedy sampling. Similar fast rates have also been established in 
privacy-sensitive~\citep{wu2025offline,weng2026improved} and competitive multi-agent~\citep{nayak2026achieving,zhang2026beyond} settings, as well as under hybrid query protocols with linear function approximation~\citep{foster2025good}.

\paragraph{Classical Minimax Lower Bound Techniques.} Information-theoretic techniques for proving non-asymptotic worst-case hardness results in \emph{offline} learning have been well-established~\citep{cover1999elements,Tsybakov2008IntroductionTN,polyanskiy2025information}, among which Le Cam's two-point method, Assouad's lemma~\citep{assouad1983deux}, and Fano's method~\citep{cover1999elements} are arguably the most commonly used~\citep{yu1997assouad}. These tools have also been extended to counterparts that mainly aim to accommodate the blessing or cost of exploration in lower bounds for \emph{interactive} protocols~\citep{lattimore2020bandit,foster2021statistical,chen2024assouad,gu2026evolution}, which are beyond our scope. Specialized to our setting, the proof of our \emph{worst case} lower bound (\Cref{thm:lowerbound-slow}) for \emph{noninteractive} learning of KL-regularized multi-armed contextual bandits builds upon ideas that are closely related to measure tilting techniques in \citet{degenne2019pure,garivier2021nonasymptotic}, where they are used to establish instance-dependent lower bounds. These techniques originate from a broader line of bandit literature that employs variants of change-of-measure arguments to prove instance-dependent hardness results in online settings or asymptotic regimes~\citep{lai1985asymptotically,mannor2004sample,audibert2010best,garivier2016optimal,degenne2019pure,garivier2021nonasymptotic}.

\section{Preliminaries}\label{sec:prelim}
In this section, we introduce the multi-armed (contextual) bandit with a KL-regularized objective, defined by a tuple $(\cS, \cA, \fgt, \eta, \piref)$. Here $\cS$ is the context space, $\cA$ is the action space and $\fgt: \cS \times \cA \to [0, 1]$ is the reward function. In the offline setting, the agent only has access to an $\iid$ dataset $\cD = \{(s_i, a_i, r_i)\}_{i=1}^n$, where $s_i$ is a context sampled from a fixed distribution $\rho\in\cP(\cS)$, $a_i \in \cA$ is the action sampled from a \emph{behavior policy}, and $r_i$ is the observed reward given by $r_i = \fgt(s_i, a_i) + \varepsilon_i$. We assume that $\varepsilon_i$ is $1$-sub-Gaussian~\citep[Definition~5.2]{lattimore2020bandit}. The learner's goal is to output a policy $\pi \in \cP(\cA|\cS)$ that maximizes the following KL-regularized objective
\begin{align}
    J(\pi) \coloneqq\EE_{(s,a) \sim \rho \times \pi} \bigg[\fgt(s,a) - \eta^{-1} \log \frac{\pi(a|s)}{\piref(a|s)}\bigg], \label{eq:kl-objective-bandit}
\end{align}
where $\piref$ is a known reference policy and $\eta^{-1}$ is proportional to the regularization intensity. For simplicity, we assume that $\piref$ is also the behavior policy that generates the dataset $\cD$. This type of ``behavior regularization'' has also been studied in previous works~\citep{zhan2022offline,zhao2026towards}. 
The unique optimal policy $\pi^*$, defined by $\pi^* \coloneqq \argmax_{\pi \in \cP(\cA | \cS)} J(\pi)$, admits the following closed form (see, e.g., \citealt[Proposition~7.16]{zhang2023ltbook}):
\begin{align}\label{eq:opt-exp}
    \pi^*(\cdot|s) \propto \piref(\cdot|s)\exp\big(\eta \cdot \fgt(s, \cdot)\big), \forall s\in\cS.
\end{align}
For any policy $\pi$, we define the \emph{suboptimality gap} as
\begin{align}\label{eq:subopt-def}
    \subopt(\pi) \coloneqq J(\pi^*) - J(\pi).
\end{align}
The goal of offline learning is to output a policy with a small suboptimality gap based on the dataset $\cD$. Throughout the paper, we characterize the rate at which the suboptimality gap decays with the sample size $n$.



\paragraph{Concentrability.} In offline RL, a standard assumption concerns the coverage of the behavior policy, which serves as a measure of the dataset's quality. Specifically, it assesses whether the dataset provides sufficient support for distributions induced by other comparator policies.

\begin{definition}[Single-Policy Concentrability] Given a reference policy $\piref$, we define the concentrability of the optimal policy $\pi^*$ with respect to $\piref$ as
\begin{align*}
    C^{\pi^*} := \sup_{s\in \cS, a \in \cA} \frac{\pistar(a|s)}{\piref(a|s)}.
\end{align*}
\end{definition}
Here we adopt the convention $0/0 = 0$. This convention is only relevant when $\piref(a|s) = 0$, in which case~\eqref{eq:opt-exp} also gives $\pi^*(a|s) = 0$, so such actions do not affect $C^{\pi^*}$.

Compared with the more stringent \emph{uniform coverage} assumptions~\citep{sidford2018near,agarwal2020model,wang2021statistical,di2024pessimistic,zhao2025sharp}, which require the data distribution of \emph{any} policy to be sufficiently covered by the dataset, the \emph{single-policy} concentrability framework~\citep{rashidinejad2021bridging,uehara2021pessimistic,xie2021bellman,rashidinejad2022optimal,cheng2022adversarially,ozdaglar2023revisiting,zhao2026towards} considered here relaxes the coverage constraint to only the distribution induced by the optimal policy, and thus typically corresponds to a much smaller concentrability coefficient. As indicated by~\citet{zhao2026towards}, this single-policy concentrability is both necessary and sufficient for learning offline KL-regularized bandits.

\paragraph{Minimax Risk.} Following~\citet{rashidinejad2021bridging}, we consider the instance class with bounded concentrability. For any instance $(r, \eta, \piref)$, let $\pi_r^*$ denote the corresponding optimal policy~\eqref{eq:opt-exp}. The class of instances with concentrability at most $C^*$ is
\begin{align*}
    \mathrm{MAB}(C^*) \coloneqq \bigg\{( r,\eta,\piref) \bigg| \sup_{s,a}\frac{\pi^*_{r}(a|s)}{\piref(a|s)} \le C^*\bigg\}.
\end{align*}
The quantity of interest in this paper is the minimax risk over this class:
\begin{align}
    \inf_{\mathtt{Alg}} \sup_{(r,\eta,\piref) \in \mathrm{MAB}(C^*)} \EE_{\cD}\big[\subopt\big(\mathtt{Alg}(\cD)\big)\big], \label{eq:minimax-risk}
\end{align}
where the infimum is taken over all algorithms $\mathtt{Alg}$ that map a dataset $\cD$ of $n$ samples to a policy, and the expectation is over the randomness of $\cD$ generated under the behavior policy $\piref$. Our upper bound (Section~\ref{sec:algorithm}) controls~\eqref{eq:minimax-risk} from above by analyzing a specific algorithm, and our hardness results (Section~\ref{sec:lower-bound}) lower bound~\eqref{eq:minimax-risk} for every algorithm.

\section{Algorithm and Suboptimality Analysis}\label{sec:algorithm}

\begin{algorithm*}[t]
\caption{Offline KL-Regularized Pessimistic Multi-armed Contextual Bandits (\algcb)}
\begin{algorithmic}[1]\label{algorithm:bandit-pess}
    \REQUIRE regularization $\eta$, reference policy $\piref$, offline dataset $\cD$

    \STATE Set $N(s,a) = \sum_{i=1}^n \ind \{(s_i, a_i) = (s,a)\}$ for all $a \in \cA$, $s \in \cS$

    \FOR{$s \in \cS$, $a \in \cA$}
    \IF{$N(s,a)=0$}
    \STATE Set the empirical reward $\fls(s,a) \leftarrow 0$, penalty $b(s,a) \leftarrow 1$
    \ELSE
    \STATE Compute the empirical reward $\fls(s,a) \leftarrow \frac{1}{N(s,a)} \sum_{i=1}^n r_i \ind \{(s_i, a_i) = (s,a)\}$
    \STATE Compute the penalty to be $b(s,a) \leftarrow \sqrt{\frac{4\log(2|\cS||\cA|/\delta)}{N(s,a)}}$, let $\fps(s,a) \leftarrow \fls(s,a) - b(s,a)$
    \ENDIF
    \ENDFOR
    \ENSURE $\hat \pi(a|s) \propto \piref(a|s) \exp \big(\eta \cdot \fps(s,a)\big)$
\end{algorithmic}
\end{algorithm*}

In this section, we introduce a variant of $\algcb$~\citep{zhao2026towards}, an algorithm for learning offline MABs with the KL-regularized objective. The algorithm is summarized in Algorithm~\ref{algorithm:bandit-pess}. In particular, for each $(s,a) \in \cS \times \cA$, we first compute the empirical average reward $\fls(s,a)$ as an estimate of the ground-truth reward function $\fgt(s,a)$. Based on $\fls(s,a)$, we then construct a pessimistic estimation of $\fgt$, which enables the algorithm to get rid of all-policy coverage as shown in~\citet{zhao2026towards}. Compared to~\citet{zhao2026towards}, we apply the standard pessimism in MABs literature~\citep{rashidinejad2021bridging,xie2021policy} given by 
\begin{align*}
    b(s,a) = \sqrt{\frac{4\log(2SA/\delta)}{N(s,a) \vee 1}},
\end{align*}
while \citet{zhao2026towards} uses reward function approximation, where the constructed pessimistic penalty has to encompass the entire function class and hence over-penalizes the reward when specialized to multi-armed cases.

We then obtain our pessimistic estimation of $\fgt(s,a)$, $\fps(s,a) = \fls(s,a) - b(s,a)$. Specifically, the following results show that $\fps$ is a pessimistic estimate with high probability and that the number of samples $N(s,a)$ does not deviate too much from its expectation $n\rho(s)\piref(a|s)$.

\begin{lemma}\label{lem:pessimisic}
Given $\delta >0$, let $\cE_1(\delta)$ denote the event that the estimation is indeed pessimistic
\begin{align}\label{eq:confbandit}
    \cE_1(\delta) \coloneqq \Big\{  \big| \fls(s,a) - \fgt(s,a)\big| \leq b(s,a) \text{, for all } (s,a) \in \cS \times \cA \Big\}.
\end{align}
We further use $\cE_2(\delta)$ to denote the event under which $N(s,a)$ does not deviate too much from the expectation, that is
\begin{align*}
    \cE_2(\delta) \coloneqq \bigg\{ \frac{1}{N(s,a) \vee 1} \leq \frac{8\log(2|\cS||\cA|/\delta)}{n\rho(s)\piref(a|s)} \text{ for all } (s,a) \in \cS \times \cA \bigg\}.
\end{align*}
Then the event $\cE_1(\delta) \cap \cE_2(\delta)$ holds with probability at least $1-\delta$.
\end{lemma}

After obtaining the pessimistic estimate $\fps$, $\algcb$ then outputs the policy using the closed-form solution $\pihat(\cdot|s) \propto \piref(\cdot | s) \exp(\eta \fps(s, \cdot))$ for all $s \in \cS$, based on the pessimistic reward $\fps$. 
The upper bound of the suboptimality of Algorithm~\ref{algorithm:bandit-pess} is given by the following theorem.

\begin{theorem}\label{thm:upperbound-bandit}
For any instance $(r, \eta, \piref) \in \mathrm{MAB}(C^*)$, with probability at least $1-\delta$, the suboptimality gap of the output of Algorithm~\ref{algorithm:bandit-pess} can be bounded as follows, depending on the regularization level:
\begin{itemize}[leftmargin=*]
    \item For large regularization $\eta^2 = \tilde{O}\big(n(SAC^{*})^{-1}\big)$, the output policy $\pihat$ obeys
    \begin{align*}
        \subopt(\pihat) = \tilde{O}\big(\eta SAC^{*} n^{-1}\big).
    \end{align*}
    \item For small regularization $\eta^2 = \tilde{\Omega}\big(n(SAC^{*})^{-1}\big)$, the output policy $\pihat$ obeys
    \begin{align*}
        \subopt(\pihat) = \tilde{O}\big(\sqrt{SAC^{*}n^{-1}}\big).
    \end{align*}
\end{itemize}
\end{theorem}


Theorem~\ref{thm:upperbound-bandit} exhibits a ``phase transition'' as regularization varies.  
When $\eta$ is small, the curvature introduced by the regularization term determines the problem, leading to an $\cO(n^{-1})$ rate, matching the results in previous literature~\citep{zhao2025sharp,zhao2026towards,foster2025good}.
On the other hand, when $\eta$ is large, the reward estimation error dominates the suboptimality gap; therefore, the problem is similar to its counterpart with the standard objective. This yields a rate of $\cO(n^{-1/2})$, recovering the rate under standard objective~\citep{rashidinejad2021bridging}.

\begin{remark}
Previously, \citet{zhao2026towards} obtained a suboptimality gap of $\tilde{O}(\eta D^2_{\pi^*} \log N_{\cR}\, n^{-1})$ under function approximation, where $D^2_{\pi^*}$ is the $D^2$-type single policy concentrability and $N_{\cR}$ is the covering number of the function class.
When restricted to the multi-armed contextual bandit setting, $D^2_{\pi^*} = \Theta(SAC^{\pi^*})$ in the worst case (Appendix~\ref{app:coverage-d2}) and $\log N_{\cR} = \tilde{\Theta}(SA)$ (Appendix~\ref{app:covering-number}), which gives an $\tilde{O}(\eta S^2A^2C^{\pi^*}n^{-1})$ suboptimality gap. Compared to their result, the suboptimality gap obtained by our version of $\algcb$ is strictly better by a factor of $SA$.
\end{remark}


\section{Hardness Results}\label{sec:lower-bound}
In this section, we lower bound the minimax risk~\eqref{eq:minimax-risk} over the class $\mathrm{MAB}(C^*)$ defined in Section~\ref{sec:prelim}. Our first theorem establishes the suboptimality gap lower bound in the large regularization regime $n = \tilde{\Omega}(\eta^2 SAC^{\pi^*})$.

\begin{theorem}\label{thm:lowerbound-fast}
For any $S \geq 1$, $A \geq 3$, $\eta > 4\log 2$, $C^* \in (2, \exp(\eta/4)]$, when the regularization level is large, i.e., $\eta^2 = \tilde{O}\big(n(SAC^{*})^{-1}\big)$, one has
\begin{align*}
    \inf_{\mathtt{Alg}} \sup_{(r,\eta,\piref) \in \mathrm{MAB}(C^*)} \EE_{\cD}\big[\subopt\big(\mathtt{Alg}(\cD)\big)\big] \gtrsim \frac{\eta SAC^*}{n\log A}.
\end{align*}
\end{theorem}

The following theorem provides a suboptimality lower bound for the small regularization regime $n = \tilde{O}(\eta^2 SAC^{\pi^*})$. Together with Theorem~\ref{thm:lowerbound-fast}, these two theorems provide a comprehensive characterization of the statistical limit of offline learning for multi-armed contextual bandits with KL regularization.

\begin{theorem}\label{thm:lowerbound-slow} For any $S \geq 1$, $A \geq 3$, $C^* \in (4,\exp(\eta/2)]$, when $n \geq 10SAC^*\log A$ and the regularization level is small, i.e., $n \leq \eta^2SAC^*/(3456\log A)$, one has
\begin{align*}
    \inf_{\mathtt{Alg}} \sup_{(r,\eta,\piref) \in \mathrm{MAB}(C^*)} \EE_{\cD}\big[\subopt\big(\mathtt{Alg}(\cD)\big)\big] \gtrsim \sqrt{\frac{SAC^*}{n\log A}}.
\end{align*}
\end{theorem}

Theorem~\ref{thm:lowerbound-slow} shows that, when $n = \tilde{O}(\eta^2SAC^{\pistar})$, the suboptimality gap of any algorithm is lower bounded by $\tilde{\Omega}(\sqrt{SAC^{\pistar}n^{-1}})$. On the other hand, when $n = \tilde{\Omega}(\eta^2SAC^{\pistar})$, Theorem~\ref{thm:lowerbound-fast} shows that for any algorithm, its suboptimality is lower bounded by $\tilde{\Omega}(\eta SAC^{\pistar}n^{-1})$. Together with the upper bounds characterized in Theorem~\ref{thm:upperbound-bandit}, these lower bounds show that $\algcb$ is nearly minimax optimal in both the large regularization and small regularization regimes.

\begin{remark}[All Ranges of $n$ and $\eta$]\label{rmk:all-regimes}
The conditions in Theorems~\ref{thm:lowerbound-fast} and~\ref{thm:lowerbound-slow} cover all ranges of $n$ and $\eta$ once these two quantities are large compared to $S$, $A$, and $C^*$. Specifically, the constraints $C^* \le \exp(\eta/4)$ and $C^* \leq \exp(\eta/2)$ hold whenever $\eta \geq 4\log C^*$. For the sample size, Theorem~\ref{thm:lowerbound-fast} applies to $n \geq c_1\eta^2SAC^*/\log A$ for an absolute constant $c_1>0$ ($c_1 = 1/512$ when $A \geq 65$; see Appendix~\ref{app:proof-lowerbound-fast}) and Theorem~\ref{thm:lowerbound-slow} applies to $10SAC^*\log A \leq n \leq \eta^2SAC^*/(3456\log A)$. In the intermediate band $\eta^2SAC^*/(3456\log A) < n < c_1\eta^2SAC^*/\log A$, the upper and lower bounds are both of order $\tilde{\Theta}(\eta^{-1})$, and a matching $\Omega(\eta^{-1})$ lower bound follows from Theorem~\ref{thm:lowerbound-fast} applied with $n' = \lceil c_1\eta^2SAC^*/\log A\rceil$ samples, since any algorithm using $n$ samples can be simulated by one that receives $n' \geq n$ samples and discards the last $n'-n$ of them. Finally, when $n < 10SAC^*\log A$, the minimax risk is $\Omega((\log A)^{-1})$ and $O(1)$, as shown in Remark~\ref{rmk:small-n} in Appendix~\ref{app:proof-lowerbound-slow}. Consequently, our results characterize~\eqref{eq:minimax-risk} up to logarithmic factors over the entire range of $n$ and $\eta$.
\end{remark}

\begin{remark}\label{rmk:kl-lower-comparison}
Previously, \citet{zhao2026towards} established an $\Omega(\eta C^{\pistar}\log N_{\cR}n^{-1})$ lower bound for large regularization and an $\Omega(\sqrt{C^{\pistar}\log N_{\cR} n^{-1}})$ lower bound for small regularization by considering a collection of $2$-armed contextual bandits. Regarding their construction, $\log N_{\cR} = \Theta(S)$, leading to $\Omega(\eta SC^{\pistar}n^{-1})$ for large regularization and $\Omega(\sqrt{SC^{\pistar} n^{-1}})$ for small regularization. As a comparison, we provide an $\tilde{\Omega}(\eta SAC^{\pistar}n^{-1})$ lower bound for large regularization in Theorem~\ref{thm:lowerbound-fast} and an $\tilde{\Omega}(\sqrt{SAC^{\pistar} n^{-1}})$ lower bound in Theorem~\ref{thm:lowerbound-slow} for small regularization. Both results strictly improve upon previous lower bounds. For both Theorem~\ref{thm:lowerbound-fast} and Theorem~\ref{thm:lowerbound-slow}, the restriction that $C^* \leq \exp(\poly (\eta))$ is inevitable, since we always have $C^{\pi^*} \leq \exp(\eta)$ in reverse KL regularized bandits with bounded rewards. Such a constraint has also appeared in previous works~\citep{zhao2026towards,foster2025good}.
\end{remark}

\section{Proof Overview of Hardness Results}\label{sec:kl-setup}


\subsection{Mechanism behind Hard Instance Construction}\label{app:overview-discussion}

The regularized objective $J(\pi) = \dotp{r}{\pi} - \eta^{-1}\KL({\pi}\|{\piref})$ combines a linear term with a KL penalty term.
Under large regularization, i.e., small $\eta$, the curvature of the dominating regularizer leads to an $n^{-1}$-type behavior.
In contrast, when $\eta$ is sufficiently large, the effect of regularization becomes negligible and the objective \emph{plausibly} approaches an unregularized one, leading to an $n^{-1/2}$-type rate similarly to that of learning standard MABs. This transition of dependency on $n$ is characterized in~\citet{zhao2026towards}, through a construction of two-point type hard instances.


However, to capture the correct dependency of $A$, we need a more curated construction, which requires a more careful investigation of the different behaviors of the regularized objective with different regularization strength.
We begin with the more tractable large-regularization regime, where $\eta$ is close to zero. In this regime, the KL regularizer dominates the objective, forcing the output policy $\pihat$ to remain close to $\pi^*$, and thus necessitating accurate estimation of the \emph{reward across all arms}.
Our hard instance, therefore, focuses on making reward estimation difficult simultaneously over a sufficiently large set of $\Omega(A)$ arms. Similar constructions have been employed in the online setting to establish the fast $\log T$-type lower bounds~\citep{ji2026nearoptimal} with correct dependency on $A$.

The construction in the small-regularization regime is more subtle. We first briefly discuss why previous two-point-type hard instances (e.g., \citealt[Theorem 15.2]{lattimore2020bandit}; \citealt[Theorem 4]{rashidinejad2021bridging}; \citealt[Theorem 2.11]{zhao2026towards}; \citealt[Theorem 5.1]{ji2026nearoptimal}) for establishing lower bounds in offline MABs fail to manifest the $A$ dependency.
In particular, this approach considers two instances that differ only on a single arm $\tilde{a}$, whose rewards differ by $\Delta$. Such a construction also requires $\tilde a$ to be the \emph{unique} optimal arm in at least one instance, implying $\piref(\tilde{a})^{-1} = C^{\pi^*}$. Since arm $\tilde a$ only receives about $n/C^{\pi^*}$ samples, no strategy can distinguish the two instances when $\Delta \lesssim \sqrt{C^{\pi^*}/n}$\footnote{Generally, $\tilde{\Theta}(\Delta^{-2})$ samples are necessary and sufficient to distinguish two distributions whose means differ by $\Delta$.}, which leads to an $\Omega(\sqrt{C^{\pi^*}/n})$ suboptimality gap and recovers the \emph{optimal} $\tilde{\Theta}(\sqrt{C^{\pi^*}/n})$ rate for the standard objective~\citep{rashidinejad2021bridging}. Importantly, under this construction, scaling up the number of arms $A$ does not change the rate, since $C^{\pi^*} = \piref(\tilde{a})^{-1}$ increases simultaneously (e.g., when $\piref=\unif(\cA)$), and thus the bound remains $A$-independent. For the KL-regularized objective with $\eta \to \infty$, the KL penalty vanishes and the regularized objective reduces to the standard objective. Since \citet{zhao2026towards} use the same type of construction for the KL-regularized objective, the $\Omega(\sqrt{C^{\pi^*}/n})$ lower bound in \citet{zhao2026towards} coincides with that of \citet{rashidinejad2021bridging}. This observation suggests that, without modifying the underlying two-point-type argument, extending the two-armed construction of \citet{zhao2026towards} to $A$-armed instances would still yield a $\tilde{\Theta}(\sqrt{SC^{\pi^*}/n})$ lower bound, leaving a gap of a factor $\sqrt{A}$ compared to the $\tilde{O}(\sqrt{SAC^{\pi^*}/n})$ upper bound.

The key obstacle is that, when there is a \emph{unique} optimal arm, the scaling of $C^{\pi^*}$ with $A$ is identical for the standard objective and for the KL-regularized objective in the limit $\eta \to \infty$. In contrast, when multiple optima exist, the relationship between $C^{\pi^*}$ and $A$ differs fundamentally between the two objectives. To illustrate this,
we analyze the behavior of the optimal policy on an $A$-armed MAB, $\piref=\unif(\cA)$, with $K$ \emph{suboptimal} arms, where the optimal set of arms is $\cA^* = \{a : r(a) = \max_{a'} r(a')\}$.
Under the standard objective, choosing a deterministic policy $\pi$ that places all mass on some $a^* \in \cA^*$ gives $C^{\pi^*}=A$.
On the other hand, under the KL-regularized objective, even as $\eta \to \infty$, the regularization still forces $\pi^* \to \piref|_{\cA^*}$, rather than any arbitrary distribution over $\cA^*$. Consequently, the concentrability $C^{\pi^*} = \sup_{a} \pi^*(a)/\piref(a) = A/(A-K)$ depends on the number of optimal arms.
This distinction leads to the following tradeoff between concentrability and hardness of learning. On one hand, since we only require the concentrability of the constructed instances to be bounded by a target constant $C^*$, the reduced value $A/(A-K)$ (as opposed to $A$) leaves more room to skew the behavior policy, potentially making the learning problem harder. On the other hand, the presence of multiple optimal arms makes the problem intrinsically easier, as identifying \emph{any} optimal arm suffices when $\eta \to \infty$. Carefully balancing this tradeoff is crucial for establishing tight lower bounds in this regime.


In summary, with large regularization, our construction prioritizes the hardness of reward estimation across all arms. With small regularization, we leverage the hardness result for learning MABs with multiple optima in Section~\ref{sec:overview-multi-armed-lower} in order to balance the tradeoff between statistical indistinguishability and concentrability, which leads to a tight lower bound detailed in \Cref{sec:kl-lower-slow-sketch}.

\subsection{Proof Outline of Theorem~\ref{thm:lowerbound-fast}}

In this section, we outline the proof of Theorem~\ref{thm:lowerbound-fast}. We first consider the simple case $\piref=\unif(\cA)$ and $S=1$.
At a high level, in the large regularization regime, since the KL-regularization term dominates, the learner is required to accurately estimate the reward on all arms to obtain a near-optimal policy. 
Motivated by this, we design the following class of hard instances, which involves many arms with unknown rewards.
Let $A_0 = \Omega(A)$ with $A_0 \leq A/2$, and let $\cV:= \{\pm 1\}^{A_0}$. We consider a class of instances $\{([1], [A], r_{\bmu}, \eta, \piref)\}_{\bmu \in \cV}$, where the reward $r_{\bmu}$ given $\bmu$ is defined as 
\begin{align*}
    r_{\bmu}(i) = \begin{cases} 1/2 +\bmu_i \Delta & \text{if } i \in [A_0], \\ 1/2 & \text{otherwise}, \end{cases}
\end{align*}
Within this class of instances, the learner has to determine the rewards on all of the first $[A_0]$ arms to achieve a near-optimal policy. In fact, we can prove that under this regime, estimation errors on each arm accumulate. In other words, the total cost is $\Omega(m\eta\Delta^2/A)$ if the learner makes $m$ estimation errors, given that an error on one arm generally costs $\Omega(\eta\Delta^2/A)$. In particular, given any $\blambda, \bmu \in\cV$, we can prove
\begin{align}
    \subopt_{\blambda}(\pihat) + \subopt_{\bmu}(\pihat) \gtrsim d_H(\bmu, \blambda)\eta\Delta^2 /A.\label{eq:overview-fast-separation}
\end{align}
Now we take $\Delta \sim \sqrt{An^{-1}}$. Since each arm only has $n/A$ samples on average, it is not enough to reliably distinguish the reward on any of the first $A_0$ arms. Therefore, the learner is expected to make at least $A_0/2$ mistakes. Formally, based on~\eqref{eq:overview-fast-separation}, we can prove the following inequality through a Fano or Assouad-type argument.
\begin{align}
    \inf_{\mathtt{Alg}} \sup_{\blambda \in \cV} \EE_{\pihat \sim \blambda} \big[\subopt_{\blambda}(\pihat)\big] \gtrsim A_0 \eta \Delta^2 /A \gtrsim \eta An^{-1}. \label{eq:overview-fast-final}
\end{align}
Finally, scaling up the bound in~\eqref{eq:overview-fast-final} by a factor $S\cdot C^{\pi^*}$ via considering instances with $S$ contexts and unbalanced $\piref$ leads to desired $\tilde{\Omega}(\eta SAC^{\pi^*}n^{-1})$ lower bound.

\subsection{Proof Outline of Theorem~\ref{thm:lowerbound-slow}}\label{sec:kl-lower-slow-sketch}
In this section, we provide an overview of the proof of Theorem~\ref{thm:lowerbound-slow}.
We first focus on the simplest setting with $\piref = \unif(\cA)$ and $|\cS| = 1$. 
For this regime,
we need the result for learning MABs with multiple optima to help us balance the tradeoff between statistical indistinguishability and concentrability. The result is stated as follows, and we defer the proof idea to Section~\ref{sec:overview-multi-armed-lower} and the formal statement to Appendix~\ref{app:multi-armed-multiple}.

\begin{lemma}[Informal]\label{lem:multi-armed-lower-informal} 
Given any $A > K > 0$ and any sufficiently large $n$, for any algorithm, there exists an \emph{unregularized} multi-armed bandit with $|\cA| = A$, $\piref = \unif(\cA)$ and $K$ suboptimal arms, whose definition is nearly the same as the regularized setting in \Cref{sec:prelim}, except that the notion of suboptimality is defined with respect to the standard objective, i.e., $J(\pi) \coloneqq \EE_{a \sim \pi}[r(a)]$,
on which the algorithm results in $\tilde{\Omega}\big(K/\sqrt{nA}\big)$ suboptimality. Moreover, this result is tight up to logarithmic factors.
\end{lemma}

We consider the following class of hard instances, which are composed of $(A-K)$ optima and parameterized as follows
\begin{align}
    \bmu \in \cV_K:= \bigg\{ \bmu \in \{\pm 1\}^A \bigg| \sum_{i=1}^A \ind(\bmu_i=-1) = K \bigg\}. \label{eq:overview-slow-hard-instance}
\end{align}
For any $\bmu \in \cV_K$, the corresponding instance is given by $([A], r_{\bmu}, \unif(\cA))$ where $r_{\bmu}(i) = 1/2 + \bmu_i \Delta$, with $\Delta > 0$ to be specified later. Since the KL-regularized objective approaches its unregularized counterpart in the small regularization regime, we have
\begin{align}
    \EE_{\blambda}\big[\subopt_{\blambda}(\pihat)\big] \gtrsim \Delta - \EE_{\blambda}\big[\Delta \la \blambda, \pihat \ra\big], \label{eq:overview-slow-reduction}
\end{align}
as long as the right-hand side of~\eqref{eq:overview-slow-reduction} is sufficiently large. In fact, the right-hand side of~\eqref{eq:overview-slow-reduction} is exactly the sub-optimality gap of instance $\blambda$ under the standard objective. Consequently, invoking Lemma~\ref{lem:multi-armed-lower-informal}
\footnote{The hard instances in~\eqref{eq:overview-slow-hard-instance} coincide with the hard instances in the proof of Lemma~\ref{lem:multi-armed-lower-informal}, enabling us to directly apply Lemma~\ref{lem:multi-armed-lower-informal} here.}
gives a lower bound of RHS of~\eqref{eq:overview-slow-reduction}. Specifically, for any $1 \leq K \leq A-1$ there exists an $\blambda \in \cV_{K}$ such that
\begin{align*}
    \EE_{\blambda}\big[\subopt_{\blambda}(\pihat)\big] \gtrsim \Delta - \EE_{\blambda}\big[\Delta \la \blambda, \pihat \ra\big] \gtrsim \sqrt{K^2 A^{-1}n^{-1}}.
\end{align*}
On the other hand, recall that we have $C^{\pi^*}=A/(A-K)$ in this case. Therefore, rewriting the lower bound with the problem parameter $C^{\pi^*}$ results in
\begin{align*}
    \EE_{\blambda}\big[\subopt_{\blambda}(\pihat)\big] \gtrsim \sqrt{C^{\pi^*}n^{-1}\cdot K^2(A-K)A^{-2}}.
\end{align*}
Now we select $K$ to balance the tradeoff between $(A-K)/A$ and $K^2/A$. A direct calculation shows that $K^2(A-K)$ takes the maximum at $K=2A/3$, at which $C^{\pi^*}= \Theta(1)$, yielding the following lower bound:
\begin{align}
    \EE_{\blambda}\big[\subopt_{\blambda}(\pihat)\big] \gtrsim \sqrt{An^{-1}}. \label{eq:overview-slow-final}
\end{align}
Finally, scaling up the bound in~\eqref{eq:overview-slow-final} by a factor of $\sqrt{S\cdot C^{\pi^*}}$ via considering instances with $S$ contexts and unbalanced $\piref$ leads to the desired $\tilde{\Omega}(\sqrt{SAC^{\pi^*}n^{-1}})$ lower bound.

\section{Proof Overview of Lemma~\ref{lem:multi-armed-lower-informal}}\label{sec:overview-multi-armed-lower}

In this section, we present the proof overview of Lemma~\ref{lem:multi-armed-lower-informal}. 
For the sake of simplicity, we restrict our output policy $\pihat$ to be deterministic in this section and therefore write $\pihat=a$ if $\pihat(a)=1$. For the general proof, please refer to Appendix~\ref{app:multi-armed-multiple}.

\paragraph{Hard Instances Construction.} We first introduce the class of bandit instances and related notation. The considered class of instances admits the following reward parameterization: 
\begin{align*}
\bmu \in \cV_K:= \bigg\{ \bmu \in \{\pm 1\}^A \bigg| \sum_{i=1}^A \ind(\bmu_i=-1) = K \bigg\},
\end{align*}
i.e., exactly $K$ entries of $\bmu$ are $-1$. For any $\bmu \in \cV_K$, the corresponding instance is given by $([A], r_{\bmu}, \unif(\cA))$ where $r_{\bmu}(i) = 1/2 + \bmu_i \Delta$, with $\Delta > 0$ to be specified later. We use $a^*(\bmu) \coloneqq \{i: \bmu_i = 1\}$ to denote the set of optimal arms of $\bmu$. We write $\bmu \sim_{i,j} \blambda$ if $a^*(\bmu) \triangle a^*(\blambda) = \{i,j\}$ with $i \in a^*(\bmu)$ and $j \in a^*(\blambda)$. We use $\PP_{\bmu}$ to denote the distribution induced by $\bmu$, and let $\cE_i = \{\pihat = i\}$.

For each pair of hard-to-distinguish instances that differ on only two arms, i.e., $\bmu \sim_{i,j} \blambda$, the induced distributions $\PP_{\blambda}$ and $\PP_{\bmu}$ are close, and thus the probability of selecting $\pihat=i$ is also similar under the two instances. However, $\pihat=i$ incurs a cost of $2\Delta$ under $\blambda$, while it is optimal under $\bmu$. Specifically, we invoke the following change-of-measure proposition.
\begin{proposition}[Proposition 18, \citealt{degenne2019pure}]\label{prop:change-of-measure}
Consider two distributions $\PP$ and $\QQ$, denoting the log-likelihood ratio with $L = \log(\ud\PP/\,\ud\QQ) $, then for any measurable event $\cE$ and threshold $\gamma \in \RR$,
\begin{align*}
    \PP(\cE) \leq e^{\gamma}\QQ(\cE) + \PP\{L > \gamma\}.
\end{align*}
\end{proposition}

\paragraph{Applying Proposition~\ref{prop:change-of-measure}.} Now we apply Proposition~\ref{prop:change-of-measure} by setting $\PP=\PP_{\bmu}$, $\QQ=\PP_{\blambda}$ and $\cE = \cE_i$, which gives the following inequality
\begin{align*}
    \PP_{\blambda}(\cE_i) \geq e^{-\gamma} \bigg[\PP_{\bmu}(\cE_i) - \PP_{\bmu}\bigg(\log \frac{\ud\PP_{\bmu}}{\ud\PP_{\blambda}}  > \gamma \bigg)\bigg].
\end{align*}
It remains to select the correct $\gamma$ to be an upper bound of the log-likelihood ratio with high probability so that we can safely get rid of the $\PP_{\bmu}\big[\ud \PP_{\bmu}/\ud \PP_{\blambda} > \gamma\big]$ term. Noticing that $\EE_{\bmu}\big[\log(\ud \PP_{\bmu}/\ud\PP_{\blambda})\big] = \kl{\PP_{\bmu}}{\PP_{\blambda}}$,
by standard concentration, we know that $\PP_{\bmu}\big[\log(\ud \PP_{\bmu}/\ud \PP_{\blambda}) > \gamma\big]$ will be small if $\gamma \gtrsim \kl{\PP_{\bmu}}{\PP_{\blambda}}$. Actually, a slightly more delicate selection of $\gamma$ enables $\PP_{\bmu}\big[\log(\ud \PP_{\bmu}/\ud \PP_{\blambda}) > \gamma\big]=o(A^{-1})$, leading to
\begin{align*}
    \PP_{\blambda}(\cE_i) \geq \exp\big(-\kl{\PP_{\bmu}}{\PP_{\blambda}}\big)\PP_{\bmu}(\cE_i) - o(A^{-1}).
\end{align*}
Moreover, to make $\PP_{\blambda}(\cE_i)$ sufficiently large, we set $\Delta = \tilde{\Theta}(\sqrt{An^{-1}})$ to make $\kl{\PP_{\bmu}}{\PP_{\blambda}} = O(1)$. Consequently, we see that the error probability is at least the same order as choosing correctly:
\begin{align}
    \PP_{\blambda}(\cE_i) \gtrsim \PP_{\bmu}(\cE_i) - o(A^{-1}) \label{eq:overview-mab-lower-pair-lower}
\end{align}

\begin{figure*}[t]
\centering   
\subfigure[]{\label{fig:lambda-to-mu}\includegraphics[width=0.42\textwidth]{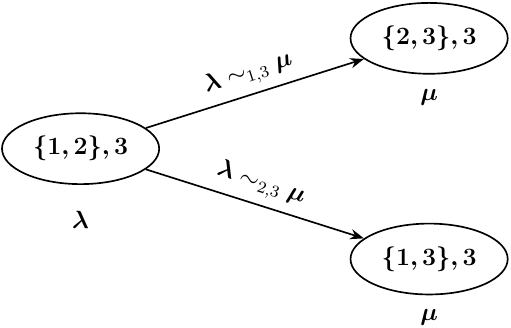}}\hspace{0.06\textwidth}
\subfigure[]{\label{fig:mu-to-lambda}\includegraphics[width=0.49\textwidth]{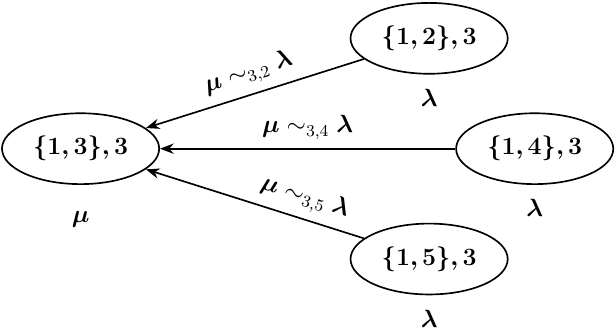}}
\caption{An illustration of the counting argument for the case $A=5$ and $K=3$. The set $\{g,h\}$ denotes the instance with its $g$-th and $h$-th arm being optimal. Each node, labeled by the tuple $(\{g,h\},k)$, corresponds to the term $\PP_{\{g,h\}}(\cE_k)$ in~\eqref{eq:overview-mab-lower-pair-lower} and each arrow indicates an instantiation of~\eqref{eq:overview-mab-lower-pair-lower}. Figure~\ref{fig:lambda-to-mu} demonstrates that each of $\PP_{\blambda}(\cE_i)$ s.t. $i \notin a^*(\blambda)$ can be bounded by $A-K$ possible $\PP_{\bmu}(\cE_i)$ s.t. $i \in a^*(\bmu)$, thus appears $A-K$ times. Likewise, Figure~\ref{fig:mu-to-lambda} shows that each $\PP_{\bmu}(\cE_i)$ on the RHS of~\eqref{eq:overview-mab-lower-pair-lower} appears $K$ times in total.}
\label{fig:counting}
\end{figure*}

\paragraph{Aggregating Everything Together.} To obtain the minimax lower bound from~\eqref{eq:overview-mab-lower-pair-lower}, we apply an ``averaging hammer'' argument~\citep{shamir2015complexity,lattimore2020bandit}
over all possible pairs $\bmu \sim_{i,j} \blambda$, where the multiplicative factor before each term is given by the following counting.
\begin{itemize}[leftmargin=*]
    \item For any $\blambda$, fix some $i \notin a^*(\blambda)$. We can find $j\in a^*(\blambda)$ with $A-K$ choices, and swapping $\blambda_i$ and $\blambda_j$ determines a unique $\bmu$. Conversely, for any $\bmu$, fix $i \in a^*(\bmu)$. We have $K$ choices of $j \notin a^*(\bmu)$ that determine a unique $\blambda$. Therefore, the factor before $\PP_{\blambda}(\cE_i)$ is $A-K$ and $\PP_{\bmu}(\cE_i)$ is $K$. See Figure~\ref{fig:counting} for a visual illustration.
    
    \item Then we count the total number of pairs. We start by constructing a tuple $\bmu \sim_{i,j} \blambda$ with selecting $K$ entries to be $-1$ to form $\bmu$, giving $\binom{A}{K}$ choices. Then we select $i \in a^*(\bmu)$ from $A-K$ choices, and subsequently $j \notin a^*(\bmu)$, leading to another $K$ choices. Together, $\bmu, i, j$ determines $\blambda$. Applying the multiplication principle, we obtain the $K(A-K)\binom{A}{K}$ factor in front of the constant term.
\end{itemize}
Applying the counting leads to the following inequality:
\begin{align*}
    (A-K)\sum_{\blambda \in \cV_K} \sum_{i \notin a^*(\blambda)} \PP_{\blambda}(\cE_i) & \gtrsim K\sum_{\bmu \in \cV_K} \sum_{i \in a^*(\bmu)}\PP_{\bmu}(\cE_i) - K(A-K)\binom{A}{K}o\bigg(\frac{1}{A}\bigg) \\
    & \gtrsim K \sum_{\bmu \in \cV_K}\bigg(1 - \sum_{j \notin a^*(\bmu)}\PP_{\bmu}(\cE_j)\bigg) - K\binom{A}{K}o(1).
\end{align*}
Using a change of variable $(\bmu, j) \to (\blambda, i)$ on the right-hand side and moving it to the left gives
\begin{align*}
    \sum_{\blambda \in \cV_K} \sum_{i \notin a^*(\blambda)} \PP_{\blambda}(\cE_i) \gtrsim \frac{K}{A}\binom{A}{K}\big(1-o(1)\big) \gtrsim \frac{K}{A}\binom{A}{K}.
\end{align*}
Therefore, there exists a $\blambda \in \cV_K$ such that 
\begin{align*}
    \subopt_{\blambda}(\pihat) = 2\Delta \sum_{i \notin a^*(\blambda)} \PP_{\blambda}(\cE_i) \gtrsim \sqrt{\frac{K^2}{An}},
\end{align*}
where we use $\Delta = \tilde \Theta(\sqrt{An^{-1}})$. This finishes the proof of Lemma \ref{lem:multi-armed-lower-informal}.

\section{Conclusion and Future Work}

In this paper, we study offline learning in KL-regularized MABs. In particular, we provide a sharp analysis of a variant of $\algcb$~\citep{zhao2026towards} across different regularization magnitudes, showing that $\algcb$ achieves a suboptimality gap of $\tilde{O}(\eta SAC^{\pi^*} n^{-1})$ under large regularization and $\tilde{O}( \sqrt{SAC^{\pi^*} n^{-1}})$ under small regularization. We further provide two lower bounds in the corresponding regimes that match the upper bounds up to logarithmic factors, certifying that $\algcb$ is nearly minimax optimal, and thereby providing a comprehensive understanding of the problem.
Currently, our analysis is limited to the multi-armed setting. Extending our results and techniques to KL-regularized objectives with richer structures, like linear or general function approximation, remains an interesting direction for future work. Moreover, our analysis assumes that the reward noise is independent across samples. Relaxing this assumption to non-independent noise, as recently studied by~\citet{abeles2025linear} for linear bandits, is another promising direction.

\appendix

\section{Reduction of Function Approximation Quantities to the Multi-armed Setting}\label{app:coverage}

In this section, we specialize the quantities appearing in the function approximation results (Table~\ref{tab:kl-regularized}) to our multi-armed setting. Appendix~\ref{app:coverage-d2} discusses the relation between the coverage notions $C^{\pi^*}$ and $D^2_{\pi^*}$, and Appendix~\ref{app:covering-number} computes the covering number $N_{\cR}$ of the tabular reward class.

\subsection{Relation Between Coverage Notions}\label{app:coverage-d2}

We provide a detailed discussion about the relation between the coverage notions $C^{\pi^*}$ and $D^2_{\pi^*}$ in~\citet{zhao2026towards} when specialized to our multi-armed setting.
We first recall the definition of $D^2$-divergence~\citep[Definition 2.5]{zhao2026towards}:
\begin{align*}
    D^2_{\cR}((s,a);\pi) = \sup_{ g,h \in \cR} \frac{\big( g(s, a) - h(s, a)\big)^2}{\E_{(s',a') \sim \rho \times \pi}[( g(s', a') - h(s', a'))^2]},
\end{align*}
where $\cR$ is the function class. When specialized to multi-armed setting, $\cR$ corresponds to all possible bounded reward functions over $\cS \times \cA$. Moreover, the single-policy $D^2$-concentrability is defined as  $D_{\pi^*}^2 \coloneqq \E_{(s,a) \sim \rho \times \pi^*} D^2_{\cR}((s,a); \piref)$~\citep[Definition 2.7]{zhao2026towards}.

In the multi-armed setting, for any $(s,a)$, the supremum in the definition of $D^2_{\cR}((s,a);\piref)$ is attained at $g = \ind\{(\cdot,\cdot) = (s,a)\}$ and $h = 0$, which gives $D^2_{\cR}((s,a);\piref) = 1/(\rho(s)\piref(a|s))$. Consequently, $D^2_{\pi^*} = \sum_{s,a}\pi^*(a|s)/\piref(a|s) \leq SAC^{\pi^*}$ for every instance. Conversely, we provide an example in which $D^2_{\pi^*} = \Omega(SAC^{\pi^*})$, which shows $D^2_{\pi^*} = \Theta(SAC^{\pi^*})$ in the worst case.

\begin{proposition}
For any $S \geq 1$, $A \geq 3$, $\eta > 0$, $2 < C \leq \exp(\eta)$, there exists an offline KL-regularized multi-armed bandit on which $C^{\pi^*} = C/2$ and $D^2_{\pi^*} = \Omega(SAC^{\pi^*})$.
\end{proposition}

\begin{proof}
We consider the following offline KL-regularized multi-armed bandit $(\cS, \cA, r, \eta, \piref)$, where $\cS = [S]$, $\cA = [A+1]$. For all $s \in \cS$, the reference policy $\piref(\cdot | s)$ is given by:
\begin{align*}
    \piref(i|s) = \frac{1}{AC}, \forall\ i \leq A; \ \piref(A+1|s) = \frac{C-1}{C},
\end{align*}
where $2 < C \leq \exp(\eta)$. We further set the reward as
\begin{align*}
    r(s,i) = 1, \forall\ i \in [A]; \ r(s,A+1) = 1-\alpha,
\end{align*}
for all $s \in \cS$ and $\alpha = \eta^{-1}\log(C-1)$. Based on the setup, since $\pi^*(\cdot|s) \propto \piref(\cdot|s) \exp(\eta r(s,\cdot))$, a direct calculation yields that for any $s \in \cS$, the optimal policy $\pi^*(\cdot|s)$ is given by
\begin{align*}
    \pi^*(i|s) = \frac{1}{2A}, \forall\ i \in [A]; \ \pi^*(A+1|s) = \frac{1}{2}.
\end{align*}
Recall that $C^{\pi^*} = \sup_{s, a}\pi^*(a|s)/\piref(a|s)$. It is easy to see that the supremum is attained at any $i \in [A]$ and $s\in \cS$, which leads to
\begin{align*}
    C^{\pi^*} = \frac{\pi^*(1|s)}{\piref(1|s)} = \frac{C}{2}.
\end{align*}
Now we come to $D^2_{\pi^*}$. 
When specialized to multi-armed setting, $\cR$ corresponds to all possible reward functions. Therefore, for any $(s,a) \in \cS \times \cA$, the supremum is achieved at $g = 1$ at $(s,a)$ otherwise $0$ and $h = 0$, which gives
\begin{align*}
    D^2_{\cR}((s,a);\piref) = \frac{1}{\rho(s)\piref(a|s)}.
\end{align*}
Recall that $D_{\pi^*}^2 \coloneqq \E_{(s,a) \sim \rho \times \pi^*}
D^2_{\cR}((s,a); \piref)$,
we know that
\begin{align*}
    D_{\pi^*}^2 = \EE_{(s,a) \sim \rho \times \pi^*} \bigg[\frac{1}{\rho(s)\piref(a|s)} \bigg] \geq \sum_{s \in \cS} \sum_{a \in [A]} \frac{\pi^*(a|s)}{\piref(a|s)} = \frac{SAC}{2} = SAC^{\pi^*},
\end{align*}
which finishes the proof.
\end{proof}

\subsection{Covering Number of the Tabular Reward Class}\label{app:covering-number}

We first recall the definition of the covering number used in the function approximation literature.

\begin{definition}[Covering Number]
Given a function class $\cR \subseteq (\cS \times \cA \to [0,1])$, a finite set $\cR(\epsilon) \subseteq \cR$ is an $\epsilon$-net of $\cR$ with respect to $\|\cdot\|_\infty$ if for any $r \in \cR$, there exists $r' \in \cR(\epsilon)$ such that $\|r - r'\|_\infty \coloneqq \max_{(s,a) \in \cS \times \cA} |r(s,a)-r'(s,a)| \leq \epsilon$. The $\epsilon$-covering number $N_{\cR}(\epsilon)$ is the smallest cardinality of such an $\cR(\epsilon)$.
\end{definition}

We now show that when $\cR$ is the tabular function class, i.e., the class of all mappings $\cS \times \cA \to [0,1]$, one has $\log N_{\cR}(\epsilon) = \Theta(SA\log(1/\epsilon))$. To do this, we first enumerate the state-action pairs as $\{(s_1,a_1),\ldots,(s_d,a_d)\}$ with $d=SA$. Then every reward function $r\in \cR$ can be identified with the vector $\mathbf{r}=(r(s_1,a_1),\ldots,r(s_d,a_d)) \in [0,1]^{d}$, and this identification is a bijection between $\cR$ and $[0,1]^d$. Under this identification,
\begin{align*}
    \|r - r'\|_{\infty} = \max_{(s,a) \in \cS \times \cA} |r(s,a)-r'(s,a)| = \|\mathbf{r} - \mathbf{r}'\|_{\infty}.
\end{align*}
Therefore, the $\epsilon$-covering number of $\cR$ is exactly the $\epsilon$-covering number of the cube $[0,1]^d$ under the $\ell_\infty$-norm~\citep[Definition~5.1]{wainwright2019high}. By~\citet[Example~5.6]{wainwright2019high}, we derive that $\log N(\epsilon; [0,1]^d, \|\cdot\|_{\infty}) = \Theta(d\log(1/\epsilon))$ for sufficiently small $\epsilon$. Substituting $d=SA$ yields
\begin{align*}
    \log N_{\cR}(\epsilon) = \log N(\epsilon; [0,1]^{SA}, \|\cdot\|_{\infty}) = \Theta\big(SA\log(1/\epsilon)\big).
\end{align*}
Hence, up to a logarithmic factor of $1/\epsilon$, we have $\log N_{\cR} = \tilde{\Theta}(SA)$, which is the reduction used in Table~\ref{tab:kl-regularized} and in the remark following Theorem~\ref{thm:upperbound-bandit}.

\section{Offline Multi-armed Bandits with Multiple Optima}\label{app:multi-armed-multiple}

\subsection{Problem Setup}

In this section, we provide a formal setup of the problem of offline multi-armed bandits with multiple optima. Recall that the problem of interest is offline MABs, which is denoted by a tuple $(\cA, r, \piref)$. Here, $\cA$ is the action space with $|\cA| = A$, and $r: \cA \to [0, 1]$ is the reward function. In the offline setting, the agent only has access to a dataset $\cD = \{(a_i, r_i)\}_{i=1}^n$. Here $a_i \in \cA$ is the action taken independently from the behavior policy $\piref$, and $r_i$ is the observed reward given by $r_i = r(a_i) + \varepsilon_i$, where $\varepsilon_i$ is independent $1$-sub-Gaussian. The goal is to find a policy that maximizes the following objective
\begin{align*}
    J(\pi) \coloneqq \EE_{a \sim \pi}[r(a)].
\end{align*}
The optimal policy $\pi^*$ is defined to be the policy that maximizes the above objective, and the suboptimality gap of a policy $\pi$ is $\subopt(\pi) \coloneqq J(\pi^*) - J(\pi)$.

In this work, we focus on the case where we might have multiple arms on which the reward takes the maximum. Specifically, we use $\cA^*$ to denote the actions that achieve the maximal reward, i.e., $\cA^* = \{ a^*\in \cA: \ r(a^*) = \max_{a\in \cA}  r(a) \}$ and $\cA_{\mathrm{sub}} = \cA \setminus \cA^*$. We use $K$ to denote the cardinality of $\cA_{\mathrm{sub}}$ and consider the non-degenerate case $1 \leq K \leq A-1$. In this section, we restrict our behavior policy to $\piref = \unif(\cA)$.

\subsection{Hardness Results}

In this section, we provide the following hardness result for this problem and the proof is deferred to Appendix~\ref{app:proof-multi-armed-lower}.

\begin{theorem}[Formal Version of Lemma~\ref{lem:multi-armed-lower-informal}]
\label{thm:multi-armed-lower}
Given any $A > K > 0$ and any sufficiently large $n$, for any algorithm, there exists a multi-armed bandit with $|\cA| = A$, uniform $\piref$ and $K$ suboptimal arms on which the algorithm results in ${\Omega}\big(K/\sqrt{nA\log A}\big)$ suboptimality.
\end{theorem}

Previously, \citet{rashidinejad2021bridging} considered a class of hard instances with a single optimal arm and used them to establish the near-optimal lower bound $\Omega(\sqrt{C^{\pi^*}/n})$, which reduces to $\Omega(\sqrt{A/n})$ under a uniform behavior policy. This result is recovered (up to logarithmic factors) as a special case of Theorem~\ref{thm:multi-armed-lower} by setting $K = A-1$. In contrast, Theorem~\ref{thm:multi-armed-lower} covers the entire range $1 \leq K \leq A-1$, and thus constitutes a nontrivial generalization of the previous result. 



\subsection{Algorithm}

In this section, we show that the rate in Theorem~\ref{thm:multi-armed-lower} cannot be improved up to minor factors. Actually, this rate can be achieved through the simplest algorithm that first computes the empirical mean and then outputs the maximizer, as shown in Algorithm~\ref{algorithm:multi-armed}.

\begin{algorithm*}[ht]
\caption{Empirical Reward Maximization for Multi-armed Bandit with Multiple Optima}
\begin{algorithmic}[1]\label{algorithm:multi-armed}
	\REQUIRE Offline dataset $\cD = \{(a_i,r_i)\}_{i=1}^n$ sampled from $\piref=\unif(\cA)$
    \FOR{all $a \in \cA$}
    \STATE Count the number of visits $N(a) = \sum_{i=1}^n \ind \{a_i = a\}$
    \IF{$N(a)=0$}
    \STATE Set the empirical reward $\fls(a) \leftarrow 0$
    \ELSE
    \STATE Compute the empirical reward $\fls(a) \leftarrow  \sum_{i=1}^n r_i \ind \{a_i = a\}/N(a)$
    \ENDIF
    \ENDFOR
    \ENSURE $\hat a = \argmax_{a\in \cA} \fls(a)$.
\end{algorithmic}
\end{algorithm*}

\begin{theorem}\label{thm:multi-armed-upper}
For any $A \geq 3$ and $n \geq 8A\log(A^2n)$, the output policy $\pihat$ of Algorithm~\ref{algorithm:multi-armed} satisfies $\EE[\subopt(\pihat)] \leq 16K\sqrt{\log A/(An)} = \tilde{O}\big(K/\sqrt{An}\big)$.
\end{theorem}

\begin{remark}
It is well known that the optimal suboptimality rate of offline contextual MABs scales as $\tilde{O}(\sqrt{SC^{\pi^*}/n})$~\citep{rashidinejad2021bridging,li2024settling}. When specialized to our setup, $S=1$ and $C^{\pi^*}=A$, leading to the $O(\sqrt{A/n})$ rate. In comparison, Theorem~\ref{thm:multi-armed-upper} establishes an $\tilde{O}(K/\sqrt{An})$ rate, which can be much smaller than $O(\sqrt{A/n})$ when $K$ is small. This comparison highlights that the existence of many optimal arms substantially simplifies offline bandit learning.
\end{remark}

Theorem~\ref{thm:multi-armed-lower} indicates an $\tilde{\Omega}\big(K/\sqrt{An}\big)$ suboptimality lower bound. On the other hand, Theorem~\ref{thm:multi-armed-upper} provides a simple algorithm attaining a suboptimality gap of $\tilde{O}(K/\sqrt{An})$, confirming that the lower bound in Theorem~\ref{thm:multi-armed-lower} is tight up to logarithmic factors. Despite the sharpened results, Algorithm~\ref{algorithm:multi-armed}, together with Theorem~\ref{thm:multi-armed-upper} and Theorem~\ref{thm:multi-armed-lower}, are restricted to uniform behavior policy. Extending these results to any $\piref$ remains an interesting direction for future work.

\paragraph{Proof Overview of Theorem~\ref{thm:multi-armed-upper}.} We provide an intuition for understanding the rate given by Theorem~\ref{thm:multi-armed-upper} and defer the full proof to~Appendix~\ref{app:proof-multi-armed-upper}. When there are only $K$ suboptimal arms, the sum of probabilities of each bad arm being selected by Algorithm~\ref{algorithm:multi-armed} is no greater than $K/A$. Consequently, the bad arms whose suboptimality gap is at most $\tilde{O}(\sqrt{A/n})$ contribute at most $\tilde{O}(K/A \cdot \sqrt{A/n}) = \tilde{O}(K/\sqrt{An})$ to the expected suboptimality.

On the other hand, since the samples are drawn uniformly over all arms, each arm receives roughly $n/A$ samples. Following the standard $\Delta^{-2}$ sample complexity for distinguishing two means separated by $\Delta$, every bad arm whose suboptimality gap exceeds $\tilde{\Omega}(\sqrt{A/n})$ can be reliably distinguished from the optimal arms and is thus rarely selected. Combining the two parts yields the $\tilde{O}(K/\sqrt{An})$ suboptimality gap.

\subsection{Additional Related Work on Bandits with Multiple Optima}

Previous works have studied MABs with multiple optimal arms in both regret minimization and pure exploration settings. In the pure exploration setting, the asymptotically optimal rate for finite-armed bandits was established by~\citet{degenne2019pure} for the fixed confidence setting. In the regret minimization setting, if the algorithm is restricted to be uniformly fast convergent~\citep[Definition 1]{garivier2019explore}, that is, to perform well across all problem instances, then the regret is given by $\sum_{a\in \cA \setminus \cA^*} \Delta_a^{-1} \log T$ asymptotically as $T \to \infty$, which roughly grows linearly with the number of suboptimal arms $K$~\citep{lai1985asymptotically,burnetas1996optimal}. This uniform convergence requirement is relaxed in a line of work on bandits with a large or even infinite number of arms ($A \gg T$ or $A = \infty$)~\citep{chaudhuri2017pac,chaudhuri2018quantile,aziz2018pure,de2021bandits,zhu2020regret}, where the dependence on $K$ would otherwise render the bounds vacuous as $A, K \to \infty$. When specialized back to the finite-armed setting, the asymptotic regret obtained becomes $A\log T \log(\Delta^{-1}) ((A-K)\Delta)^{-1}$~\citep{de2021bandits}, which is tight up to the $\log(\Delta^{-1})$ factor. From a minimax perspective, it is possible to achieve an $O(\sqrt{AT/(A-K)})$ regret~\citep{zhu2020regret}, which is optimal when $K \geq A/2$, as implied by the two-point argument in \citet[Theorem 15.2]{lattimore2020bandit}, as remarked in~\citet{zhu2020regret}. 
Despite these advances, the sample complexity of this problem in the pure offline setting remains largely underexplored. Furthermore, in both online and offline regimes, the minimax optimal rate for the regime $K < A/2$, where the Le Cam two-point method becomes ineffective, is still not well understood.

\section{Proof of Theorems in Section~\ref{sec:algorithm}}

\subsection{Proof of Lemma~\ref{lem:pessimisic}}

\begin{proof}[Proof of Lemma~\ref{lem:pessimisic}]
The proof is standard in previous literature (e.g., Lemma B.1 in \citet{xie2021policy}), and we provide the detailed proof here for completeness. We first prove that $\cE_{1}$ holds with high probability. Fixing a pair $(s,a)$, the event holds trivially when $N(s,a) = 0$. When $N(s,a) \geq 1$, conditioning on $N(s,a)$ and applying Azuma-Hoeffding's inequality (Lemma~\ref{lem:azuma-hoeffding}) to both tails, we know that with probability at least $1-\delta/(2SA)$, we have
\begin{align*}
   \bigg|\fgt(s,a) -  \frac{1}{N(s,a)} \sum_{i=1}^n r_i \ind \{(s_i, a_i) = (s,a)\}\bigg| \leq \sqrt{\frac{2\log(4SA/\delta)}{N(s,a)}} \leq b(s,a),
\end{align*}
where the last inequality holds since $4SA/\delta \leq (2SA/\delta)^2$.
Taking a union bound over all $(s,a) \in \cS \times \cA$, we obtain that $\cE_1$ holds with probability at least $1 - \delta/2$. For the second event, we directly invoke Lemma~\ref{lem:binary-concentration} and obtain that for any fixed $(s,a)$ with probability at least $1 - \delta/(2SA)$, 
\begin{align*}
    \frac{1}{N \vee 1} \leq \frac{8\log(2SA/\delta)}{n \rho(s) \piref(a|s)}.
\end{align*}
Taking union bound over all $(s,a) \in \cS \times \cA$ gives that $\cE_2$ holds with probability at least $1 - \delta/2$. Taking union bound once more over $\cE_1$ and $\cE_2$ finishes the proof.
\end{proof}

\subsection{Proof of Theorem~\ref{thm:upperbound-bandit}}

In this section, we present the proof of Theorem~\ref{thm:upperbound-bandit}. First, we need the following regret decomposition lemma for regularized objective.

\begin{lemma}[Lemma 2.16, \citealt{zhao2026towards}]\label{lem:taylor-expansion} Let $r: \cS \times \cA \to \RR$ be any reward function, and $\pi_r(\cdot|s) \propto \piref(\cdot |s) \exp(\eta r(s,\cdot))$ be the optimal policy under $r$, then there exists some $\gamma \in [0,1]$ such that if we denote $r_\gamma = \gamma r + (1-\gamma) r^*$ and $\pi_{\gamma}(\cdot|s) \propto \piref(\cdot |s) \exp(\eta r_{\gamma}(s,\cdot))$, then
\begin{align*}
    J(\pi^*) - J(\pi_r) \leq \eta \EE_{(s,a) \sim \rho \times \pi_{\gamma}}\big[\big(r^*(s,a) - r(s,a)\big)^2\big].
\end{align*}
\end{lemma}

Moreover, Lemma 2.15 in~\citet{zhao2026towards} shows that when $r \leq r^*$, we can take $\gamma=0$. For simplicity, we combine the original Lemma 2.15 in~\citet{zhao2026towards} with Lemma~\ref{lem:taylor-expansion} and state the following Lemma.

\begin{lemma}\label{lem:taylor-expansion-pessimism}
Let $r: \cS \times \cA \to \RR$ be any reward function satisfying $r(s,a) \leq r^*(s,a)$ for any $(s,a) \in \cS \times \cA$, and let $\pi_r(\cdot|s) \propto \piref(\cdot |s) \exp(\eta r(s,\cdot))$ be the optimal policy under $r$, then
\begin{align*}
J(\pi^*) - J(\pi_r) \leq \eta \EE_{(s,a) \sim \rho \times \pistar}\big[\big(r^*(s,a) - r(s,a)\big)^2\big].
\end{align*}
\end{lemma}

Now we are ready to prove Theorem~\ref{thm:upperbound-bandit}.

\begin{proof}[Proof of Theorem~\ref{thm:upperbound-bandit}]
We derive the proof conditioning on event $\cE_1(\delta) \cap \cE_2(\delta)$. $\cE_1(\delta)$ ensures that for all $s,a$, $\fps(s,a) \leq \fgt(s,a)$. We first prove the fast-rate bound, where the proof is almost identical to the proof in~\citet{zhao2026towards} and we provide the full proof for completeness. First, by Lemma~\ref{lem:taylor-expansion-pessimism}, we have the following regret decomposition
\begin{align*}
    J(\pi^*) -  J(\hat \pi) \leq \eta \EE_{(s,a) \sim \rho \times \pi^*}\big[\big(\fps(s,a) - \fgt(s,a)\big)^2\big] \leq 4\eta\EE_{(s,a) \sim \rho \times \pi^*} [b^2(s,a)],
\end{align*}
where the second inequality holds due to $|\fps(s,a) - \fgt(s,a)| \leq 2b(s,a)$ on $\cE_1$. Plugging in the exact construction of $b(s,a)$, we know that 
\begin{align*}
    J(\pi^*) -  J(\hat \pi) & \leq 4 \eta \sum_{s \in \cS} \rho(s) \sum_{a \in \cA} \pi^*(a|s)\frac{4\log(2|\cS||\cA|/\delta)}{N(s,a) \vee 1} \\
    & \leq 4 \eta \sum_{s \in \cS} \rho(s) \sum_{a \in \cA} \pi^*(a|s)\frac{32\log^2(2|\cS||\cA|/\delta)}{n\rho(s)\piref(a|s)} \\
    & \leq 128 \eta C^{\pi^*} \log^2(2|\cS||\cA|/\delta) \sum_{(s,a) \in \cS \times \cA}\rho(s)\piref(a|s) \frac{1}{\rho(s)\piref(a|s)} \\
    & = \tilde{O}(\eta C^{\pi^*} SAn^{-1}),
\end{align*}
where the second inequality holds due to event $\cE_2$ and the last inequality holds due to the definition $\sup_{s,a} \pi^*(a|s) / \piref(a|s) = C^{\pi^*}$. Since $(r,\eta,\piref) \in \mathrm{MAB}(C^*)$ implies $C^{\pi^*} \leq C^*$, this gives $\tilde{O}(\eta SAC^* n^{-1})$. 

We then prove the slow-rate bound, which follows the standard argument in the analysis. In particular, we have
\begin{align*}
    J(\pi^*) - J(\hat \pi) & \leq \EE_{s, a\sim \rho \times \pi^*} [\fgt(s,a) - \fps(s,a)] \\
    & \leq 4 \sum_{s \in \cS} \rho(s) \sum_{a \in \cA} \pi^*(a|s) \sqrt{\frac{\log(2|\cS||\cA|/\delta)}{N(s,a) \vee 1}} \\
    & \leq 16 \sum_{s \in \cS} \rho(s) \sum_{a \in \cA} \pi^*(a|s)\frac{\log(2|\cS||\cA|/\delta)}{\sqrt{n\rho(s)\piref(a|s)}} \\
    & = 16 \log(2|\cS||\cA|/\delta) \sum_{s \in \cS} \rho(s) \sum_{a \in \cA} \pi^*(a|s)\frac{1}{\sqrt{n\rho(s)\pi^*(a|s)}} \sqrt{\frac{\pi^*(a|s)}{\piref(a|s)}} \\
    & \leq 16 \log(2|\cS||\cA|/\delta)\sqrt{C^{\pi^*}n^{-1}} \sum_{s \in \cS} \rho(s) \sum_{a \in \cA}\pi^*(a|s)\frac{1}{\sqrt{\rho(s)\pi^*(a|s)}} \\
    & = \tilde{O}\big(\sqrt{SAC^{\pi^*}n^{-1}}\big) = \tilde{O}\big(\sqrt{SAC^{*}n^{-1}}\big),
\end{align*}
where the second inequality holds due to event $\cE_1(\delta)$, the third holds due to event $\cE_2(\delta)$, the fourth holds due to the definition of $C^{\pi^*}$, the subsequent equality holds due to the Cauchy-Schwarz inequality $\sum_{s,a}\sqrt{\rho(s)\pi^*(a|s)} \leq \sqrt{SA}$, and the last holds due to $C^{\pi^*} \leq C^*$. By Lemma~\ref{lem:pessimisic}, we know that $\cE_1(\delta) \cap \cE_2(\delta)$ holds with probability at least $1-\delta$, which finishes the proof.
\end{proof}

\section{Proof of Theorems in Section~\ref{sec:lower-bound}}

\subsection{Proof of Theorem~\ref{thm:lowerbound-fast}}\label{app:proof-lowerbound-fast}

\begin{proof}[Proof of Theorem~\ref{thm:lowerbound-fast}] We distinguish two cases depending on the size of the action set. The main argument below handles the case $|\cA| - 1 \geq 64$; the remaining case of a bounded action set, $3 \leq |\cA| \leq 64$, is handled at the end of the proof.

\noindent\textbf{Case 1: $|\cA| - 1 \geq 64$.} We first fix the context size $S \geq 1$, action set size $A+1$ with $A \geq 64$, regularization parameter $\eta > 4\log 2$, upper bound of coverage coefficient $C^* \in (2, \exp(\eta/4)]$ and any $n \geq \eta^2SC^* A(512\log A)^{-1}$, we consider the family of contextual bandits with $S \coloneqq |\cS|, A \coloneqq |\cA| - 1 < \infty$ and reward function in some function class $\cV$ composed of function $\cS \times \cA \to [0, 1]$ as follows.
\begin{align}
    \mathrm{CB}_{\cV} \coloneqq \{(\cS, \cA, \rho, r, \piref, \eta): r \in \cV, \rho \in \cP(\cS), \piref \in \cP(\cA|\cS) \}. \label{eq:contextual-bandit-class} 
\end{align}
We set $\cS = [S]$, $\cA = [A+1]$, $\rho = \mathsf{Unif}(\cS)$, and the reference policy to be
\begin{align*}
  \forall s \in \cS, \piref(i|s) = (CA)^{-1} \ \forall i  \in [A], \piref(A+1|s) = 1-C^{-1},
\end{align*}
where $C \ge 1$ is a parameter to be specified later. Now we construct our reward functions. We first leverage \Cref{lem:max-signals} with $\Sigma = \{-1, +1\}$ and obtain a set $\cU \subset \{-1,+1\}^A$ such that (1) $|\cU| \geq \exp(A/8)$ and (2) for any $\bmu, \bmu' \in \cU, \bmu \neq \bmu'$, one has $\|\bmu - \bmu'\|_1 \geq A/2$, i.e., $d_H(\bmu, \bmu') \geq A/4$. We would like to apply the general version of \Cref{lem:max-signals} once again with $\Sigma = \cU$. However, the construction does not guarantee that $|\cU|$ is a prime power. To address this, we restrict $\cU$ to a sufficiently large subset whose cardinality is prime, and hence a prime power. Specifically, Bertrand's postulate states that for every natural number $k \geq 2$, there exists a prime $q \in [k, 2k]$. Applying this result with $k = \lfloor |\cU|/2 \rfloor$, we can select a prime $q \in [\lfloor |\cU|/2 \rfloor, |\cU|]$ and a subset $\cU' \subseteq \cU$ satisfying $|\cU'| = q$. Since $A \geq 64$, we have $|\cU| \geq \exp(8)$ and thus $|\cU'| \geq \lfloor |\cU|/2 \rfloor \geq |\cU|/3$. Now we invoke \Cref{lem:max-signals} with $\Sigma = \cU'$ and obtain a subset $\cV \subseteq (\cU')^{S}$ such that
\begin{enumerate}
    \item $\log_{|\cU'|}|\cV| \geq S/8$, which gives $\log |\cV| = \log_{|\cU'|}|\cV| \log |\cU'| \geq S/8\cdot(A/8 - \log 3) \geq SA/96$;
    \item For any $\bnu, \bnu' \in \cV, \bnu \neq \bnu'$, one has $d_H(\bnu, \bnu') \geq S/2$.
\end{enumerate}
Moreover, we may assume without loss of generality that $|\cV| \geq |\cU'|$: otherwise, we replace $\cV$ by $\{(\bmu, \ldots, \bmu): \bmu \in \cU'\}$, whose distinct elements have Hamming distance $S$, so that both properties above still hold. Consequently, we also have $\log |\cV| \geq \log |\cU'| \geq A/8 - \log 3 \geq A/12$.

Now we construct our Bernoulli reward function whose means are given as follows.
\begin{align*}
    \cG = \{ r_{\bnu}(s, i) = 1/2 + \bnu_{s,i}\Delta \ \forall i \in [A], r_{\bnu}(s, A+1) = 1/2 - \alpha, \ \forall s \in \cS  | \bnu \in \cV\},
\end{align*}
where $\bnu_{s,i}$ is the $i$-th coordinate of the $s$-th entry of $\bnu$. 
In other words, for any two reward function $r_1$ and $r_2$ in $\cG$, there exist $\Omega(S)$ contexts such that for these contexts, there exist $\Omega(A)$ arms such that the two rewards on these arms are different.

Now we consider two different reward functions $r_1, r_2 \in \cG$ for a certain shared context $s$ such that in this context the two reward functions differ on at least $A/4$ arms. In the following, we drop the subscripts for context $s$ when there is no ambiguity. Without loss of generality, we assume that $r_1$ and $r_2$ are distributed as follows.
\begin{align*}
    & r_1(i) = 1/2 + \Delta, r_2(i) = 1/2 - \Delta, \ \forall i \in [1,l]; \\
    & r_1(i) = 1/2 - \Delta, r_2(i) = 1/2 + \Delta, \ \forall i \in [l+1,m]; \\
    & r_1(i) = r_2(i) = 1/2 + r^*(i), r^*(i) \in \{\pm \Delta\} , \ \forall i \in [m+1, A],
\end{align*}
where $0 \leq l \leq m$ and $m \geq A/4$ are some integers. Let  $\pi^*_1$ and $\pi^*_2$ be the optimal policy regarding $r_1$ and $r_2$, we have
\begin{align*}
    \pi^*_1(i) = \frac{\piref(i)\exp(\eta r_1(i))}{\sum_{j=1}^{A+1}\piref(j)\exp(\eta r_1(j))}, \ \pi^*_2(i) = \frac{\piref(i)\exp(\eta r_2(i))}{\sum_{j=1}^{A+1}\piref(j)\exp(\eta r_2(j))}.
\end{align*}
Now we assign $C^* = C$ and $\alpha = \eta^{-1}\log (C-1)$. Therefore, we know that 
\begin{align}
    C^{\pi^*_1} & = \max_i \frac{\exp(\eta r_1(i))}{\sum_{j=1}^{A+1}\piref(j)\exp(\eta r_1(j))} \notag \\
    & \leq \frac{\exp(\eta \Delta)}{\exp(-\eta\Delta)/C + (C-1)\exp(-\eta\alpha)/C} \notag \\
    & = \frac{C\exp(\eta\Delta)}{\exp(-\eta\Delta) + 1} \notag \\
    & \leq C \notag \\
    & = C^*, \label{eq:coverage-upper}
\end{align}
where the last inequality holds whenever $\exp(\eta\Delta) \leq (1+\sqrt{5})/2$, which is guaranteed by our choice of $\Delta$ below (see~\eqref{eq:fast-eta-delta}).
The argument above holds similarly for $C^{\pi^*_2}$, which gives that for all reward function $r \in \cG$, we have $C^{\pi^*_r} \leq C^*$. Now we consider the suboptimality gap $\subopt_1(\hat \pi) = \subopt(\hat \pi;r_1)$ and $\subopt_2(\hat \pi) = \subopt(\hat \pi;r_2)$. By Lemma~\ref{lem:kl-subopt-decompose}, we know that
\begin{align*}
    \subopt_1(\hat \pi) + \subopt_2(\hat \pi) = \eta^{-1} \kl{\hat \pi}{\pi^*_1} + \eta^{-1} \kl{\hat \pi}{\pi^*_2}.
\end{align*}
Similar to Eq.~(B.9) in~\citet{zhao2026towards}, $\subopt_1(\hat \pi) + \subopt_2(\hat \pi)$ is minimized by $\bar \pi$ with $\bar \pi(i) \propto \sqrt{\pi^*_1(i)\pi^*_2(i)}$. Consequently, we have
\begin{align*}
    \subopt_1(\hat \pi) + \subopt_2(\hat \pi) \geq \subopt_1(\bar \pi) + \subopt_2(\bar \pi) = \eta^{-1} \kl{\bar \pi}{\pi^*_1} + \eta^{-1} \kl{\bar \pi}{\pi^*_2},
\end{align*}
which can be reformulated as
\begin{align*}
    & \subopt_1(\bar \pi) + \subopt_2(\bar \pi) \\
    & \quad = 2 \eta^{-1} \sum_{a \in \cA} \bar \pi(a) \log \frac{\sqrt{\sum_{i=1}^{A+1} \piref(i)\exp(\eta r_1(i))}\sqrt{\sum_{j=1}^{A+1} \piref(j)\exp(\eta r_2(j))}}{\sum_{k=1}^{A+1} \piref(k)\exp\Big(\eta \big(r_1(k) + r_2(k)\big)/2\Big)} \\
    & \quad = 2 \eta^{-1}  \log \frac{\sqrt{\sum_{i=1}^{A+1} \piref(i)\exp(\eta r_1(i))}\sqrt{\sum_{j=1}^{A+1} \piref(j)\exp(\eta r_2(j))}}{\sum_{k=1}^{A+1} \piref(k)\exp\Big(\eta \big(r_1(k) + r_2(k)\big)/2\Big)} \\
    & \quad = \eta^{-1} (X_1 + X_2),
\end{align*}
where
\begin{align*}
    X_1 & = \log \frac{\sum_{j=1}^{A+1} \piref(j)\exp(\eta r_1(j))}{\sum_{i=1}^{A+1} \piref(i)\exp\Big(\eta \big(r_1(i) + r_2(i)\big)/2\Big)}; \\
    X_2 & = \log \frac{\sum_{j=1}^{A+1} \piref(j)\exp(\eta r_2(j))}{\sum_{i=1}^{A+1} \piref(i)\exp\Big(\eta \big(r_1(i) + r_2(i)\big)/2\Big)}.
\end{align*}
Now we compute the first term $X_1$.
\begin{align*}
    X_1 & = \log \frac{\displaystyle \sum_{j=1}^l \piref(j)e^{\eta\Delta} + \sum_{j=l+1}^m \piref(j)e^{-\eta\Delta} + \sum_{j=m+1}^A \piref(j)e^{\eta r^*(j)} + \piref(A+1)e^{-\eta\alpha}}{\displaystyle \sum_{j=1}^m \piref(j) + \sum_{j=m+1}^A \piref(j)e^{\eta r^*(j)} + \piref(A+1)e^{-\eta\alpha}} \\
    & = \log \frac{(CA)^{-1}le^{\eta \Delta} + (CA)^{-1}(m-l)e^{-\eta\Delta} + (CA)^{-1}\sum_{j=m+1}^A e^{\eta r^*(j)} + C^{-1}(C-1) e^{-\eta \alpha} }{(CA)^{-1}m + (CA)^{-1}\sum_{j=m+1}^A e^{\eta r^*(j)} + C^{-1}(C-1) e^{-\eta \alpha}} \\
    & = \log \frac{A^{-1}l\exp(\eta \Delta) + A^{-1}(m-l) \exp(-\eta \Delta) + A^{-1}\sum_{j=m+1}^A \exp(\eta r^*(j)) + (C-1)e^{-\eta \alpha}}{A^{-1}m + \underbrace{A^{-1}\sum_{j=m+1}^A \exp(\eta r^*(j)) + (C-1)e^{-\eta \alpha}}_{M}} \\
    & = \log \frac{A^{-1}l\exp(\eta \Delta) + A^{-1}(m-l) \exp(-\eta \Delta) + M}{A^{-1}m + M}.
\end{align*}
Following a similar argument, we know that 
\begin{align*}
    X_2 = \log \frac{A^{-1}(m-l)\exp(\eta \Delta) + A^{-1}l \exp(-\eta \Delta) + M}{A^{-1}m + M}
\end{align*}
Since $X_1 + X_2$ is a concave function of $l$, it takes the minimum with respect to $l$ when $l=0$ or  $l=m$. Therefore we have
\begin{align*}
    X_1 + X_2 & \geq \log \frac{A^{-1}m\exp(\eta \Delta)  + M}{A^{-1}m + M} + \log \frac{A^{-1}m\exp(-\eta \Delta)  + M}{A^{-1}m + M} \\
    & = \log \frac{M^2 + A^{-2}m^2 + MA^{-1}m(\exp(\eta \Delta) + \exp(- \eta \Delta))}{(A^{-1}m + M)^2} \\
    & = \log \Bigg(1 + \frac{2 MA^{-1}m}{(A^{-1}m + M)^2}\bigg(\frac{\exp(\eta \Delta) + \exp(- \eta \Delta)}{2} - 1\bigg) \Bigg) \\
    & \geq \log \Bigg(1 + \frac{M}{2(1 + M)^2}\bigg(\frac{\exp(\eta \Delta) + \exp(- \eta \Delta)}{2} - 1\bigg) \Bigg),
\end{align*}
where the last inequality holds since $x \mapsto x(x+M)^{-2}$ is increasing on $[0, M]$, $M \geq 1$ (shown below), and $1/4 \leq A^{-1}m \leq 1$. Therefore, we have
\begin{align}
    \subopt_1(\bar \pi) + \subopt_2(\bar \pi) \geq \eta^{-1} \log \Bigg(1 + \frac{M}{2(1 + M)^2}\bigg(\frac{\exp(\eta \Delta) + \exp(- \eta \Delta)}{2} - 1\bigg) \Bigg). \label{eq:lower-subopt-all}
\end{align}
Now we aim to lower bound \eqref{eq:lower-subopt-all}. We pick $\Delta = \sqrt{SAC^*n^{-1} \log^{-1}A}/64$. We know that
\begin{align}
    n \geq \frac{\eta^2SC^* A}{512 \log A} \quad \Rightarrow \quad \eta \Delta \leq \frac{\sqrt{512}}{64} \leq 0.36 < \log\frac{1+\sqrt{5}}{2}, \label{eq:fast-eta-delta}
\end{align}
which also implies $\Delta \leq 0.36/\eta \leq 1/4$ since $\eta > 4\log 2$.
Now, the lower bound of $M$ is straightforward as follows
\begin{align*}
    M = A^{-1}\sum_{j=m+1}^A \exp(\eta r^*(j)) + (C-1)e^{-\eta \alpha} = A^{-1}\sum_{j=m+1}^A \exp(\eta r^*(j)) + 1 \geq 1.
\end{align*}
On the other hand, since $A - m \leq 3A/4$, we can upper bound $M$ as follows:
\begin{align*}
    M = A^{-1}\sum_{j=m+1}^A \exp(\eta r^*(j)) + 1
    \leq 1 + \frac{3}{4}\exp(\eta\Delta).
\end{align*}
Therefore, we know that
\begin{align*}
    \frac{M}{(1+M)^2} \geq \frac{1 }{(1+M)^2} \geq \frac{1}{(2+3\exp(\eta \Delta)/4)^2} \geq \frac{1}{16}.
\end{align*}
Recall that  $(\mathrm{e}^x + \mathrm{e}^{-x})/2 - 1 = x^2\sum_{k=0}^\infty \frac{x^{2k}}{(2k+2)!} \geq x^2/2$ for all $x \in \RR$, which implies that
\begin{align*}
    \subopt_1(\hat \pi) + \subopt_2(\hat \pi) & \geq \eta^{-1} \log \bigg(1 + \frac{M}{4(1 + M)^2} \eta^2\Delta^2\bigg) \geq \eta^{-1} \log \bigg(1 +\frac{1}{64} \eta^2\Delta^2\bigg).
\end{align*}
Since $\eta\Delta \leq 1$, we know that $\eta^2 \Delta^2 /64 \leq 1$. By the fact that $\log(1+x) \geq x/2$ for all $x \in [0,1]$, we know that
\begin{align*}
     \subopt_1(\hat \pi) + \subopt_2(\hat \pi) \geq \eta^{-1} \frac{\eta^2 \Delta^2}{128} = \frac{1}{128}\eta\Delta^2.
\end{align*}

Now we consider all contexts $s \in \cS$. Consider two different reward functions $r_1$ and $r_2$, then for any $\hat \pi: \cS \to \cP(\cA)$
\begin{align*}
    & \subopt(\hat \pi, r_1) +  \subopt(\hat \pi, r_2) \\
    & \quad = \frac{1}{S}\sum_{s \in \cS}\big(\subopt_s(\hat \pi(\cdot | s), r_1(s, \cdot)) + \subopt_s(\hat \pi(\cdot | s), r_2(s, \cdot))\big) \\
    & \quad \geq \frac{1}{256}\eta\Delta^2,
\end{align*}
where the last inequality holds due to our construction that on at least $S/2$ contexts $r_1$ and $r_2$ differ on at least $A/4$ actions. Next, we bound the KL divergence between the data distributions. Since each sample satisfies $s \sim \unif(\cS)$, $a \sim \piref(\cdot|s)$, and the reward distributions under $r_1$ and $r_2$ only differ on arms in $[A]$ with means in $[1/4, 3/4]$, using $\kl{\mathsf{Bern}(p)}{\mathsf{Bern}(q)} \le (p-q)^2/\big(q(1-q)\big)$, we have
\begin{align*}
    \kl{P_{\cD_{r_1}} }{P_{\cD_{r_2}}} \leq n \cdot \frac{1}{S}\sum_{s \in \cS}\sum_{i \in [A]} \frac{1}{CA}\cdot \frac{(2\Delta)^2}{3/16} = \frac{64n\Delta^2}{3C^*} = \frac{SA}{192\log A}.
\end{align*}
Now we invoke Fano's inequality (Lemma~\ref{lem:fano}) and use $\log|\cV| \geq \max\{SA/96, A/12\}$ to obtain
\begin{align*}
     \inf_{\pi}\sup_{\mathrm{inst} \in \mathrm{CB}_{\cG}}  \subopt(\hat\pi; \mathrm{inst})
     & \geq \frac{1}{512}\eta\Delta^2 \Bigg(1 - \frac{\max_{r_1\neq r_2 \in \cG}\kl{P_{\cD_{r_1}} }{P_{\cD_{r_2}}} + \log 2}{\log |\cV|} \Bigg) \\
     & \geq \frac{1}{512}\eta\Delta^2\Bigg(1 - \frac{1}{2\log A} - \frac{12\log 2}{A} \Bigg) \geq \frac{3}{2048}\eta\Delta^2,
\end{align*}
where the last inequality holds due to $A \geq 64$. Recall that $\Delta = \sqrt{SAC^*n^{-1}\log^{-1}A}/64$, and thus for any $\hat{\pi}$ we have
\begin{align*}
    \sup_{\mathrm{inst} \in \mathrm{CB}_{\cG}} \subopt(\hat\pi; \mathrm{inst}) \gtrsim \frac{\eta C^* SA}{n\log A},
\end{align*}
which finishes the proof of Case 1. Since $|\cA| = A+1 \leq 2A$, this bound is of the same order when expressed in terms of $|\cA|$.

\noindent\textbf{Case 2: $3 \leq |\cA| \leq 64$.} In this case, the action-space size is bounded by a universal constant, and the desired bound follows from the two-armed construction of~\citet[Theorem~2.11]{zhao2026towards}, which establishes an $\Omega(\eta SC^*n^{-1})$ lower bound when $\eta^2 = O(n(SC^*)^{-1})$. These two-armed instances can be embedded into an $|\cA|$-armed problem by retaining the informative first arm and replacing the second arm with $|\cA|-1$ statistically identical copies. For example, if the original reward vector and reference policy are
\begin{align*}
    r(s,\cdot) = (1/2 \pm \Delta, 1/2 - \alpha), \quad \piref(\cdot|s) = (C^{-1}, 1 - C^{-1}),
\end{align*}
we define the corresponding $|\cA|$-armed instance by
\begin{align*}
    r(s,\cdot) = ( 1/2 \pm \Delta, 1/2 - \alpha, 0, \ldots, 0), \quad \piref(\cdot|s) = (C^{-1}, 1 - C^{-1}, 0, \ldots, 0).
\end{align*}
Any policy assigning positive probability to one of the additional actions incurs infinite KL divergence with respect to $\piref$, and is therefore irrelevant to the KL-regularized optimization problem; moreover, these actions never appear in the dataset. Hence, the embedded $|\cA|$-armed instance is equivalent to the original two-armed instance, and the proof of~\citet[Theorem~2.11]{zhao2026towards} applies verbatim, yielding an $\Omega(\eta SC^*n^{-1})$ lower bound when $\eta^2 = O(n(SC^*)^{-1})$. Since $3 \leq |\cA| \leq 64$, we have $\Omega(\eta SC^*n^{-1}) = \Omega(\eta S|\cA|C^*n^{-1})$, and the condition $\eta^2 = O(n(SC^*)^{-1})$ is equivalent to $\eta^2 = O(n(S|\cA|C^*)^{-1})$.

Combining the two cases establishes Theorem~\ref{thm:lowerbound-fast} over its full stated range $A \geq 3$.
\end{proof}

\subsection{Proof of Theorem~\ref{thm:lowerbound-slow}}\label{app:proof-lowerbound-slow}

\begin{proof}[Proof of Theorem~\ref{thm:lowerbound-slow}] We distinguish two cases depending on the size of the action set. The main argument below handles the case $|\cA| \geq 7$; the remaining case of a bounded action set, $3 \leq |\cA| \leq 6$, is handled at the end of the proof.

\noindent\textbf{Case 1: $|\cA| \geq 7$.} The proof largely extends from the proof of Theorem~\ref{thm:multi-armed-lower}. Throughout this case, with a slight abuse of notation, $A$ and $C$ denote the parameters of the construction below, which are related to the quantities in Theorem~\ref{thm:lowerbound-slow} through $A = \lfloor (|\cA|-1)/2 \rfloor \geq 3$ and $C^* = 2C$. The construction uses $2A+1$ actions. When $|\cA|$ is even, i.e., $|\cA| = 2A+2$, we additionally include a dummy action with reference probability $0$ and reward $0$; as argued in Case~2 of the proof of Theorem~\ref{thm:lowerbound-fast}, such an action is irrelevant to the KL-regularized problem and never appears in the dataset, and hence we omit it below. Since $A \leq |\cA|/2$, we have $SAC\log A \leq S|\cA|C^*\log|\cA|/4$, so the condition $n \geq 10S|\cA|C^*\log|\cA|$ in Theorem~\ref{thm:lowerbound-slow} implies $n \geq 10SAC\log A$. Similarly, since $A \geq (|\cA|-2)/2 \geq |\cA|/3$ for $|\cA| \geq 6$, we have $SAC/\log A \geq S|\cA|C^*/(6\log|\cA|)$, so the condition $n \leq \eta^2S|\cA|C^*/(3456\log|\cA|)$ implies $n \leq \eta^2SAC/(576\log A)$. Given the context size $S$, size of action set $|\cA| = 2A+1$ and upper bound of coverage coefficient $C \in [2, e^{\eta/2}/2]$, we consider the KL-regularized contextual bandit (with Gaussian reward) class parameterized by $\bmu \in \RR^{S\times 2A}$, such that for any $s \in \cS$, $\bmu(s) \in \{1/2, 1/2+\Delta\}^{2A}$ and half of the entries of $\bmu(s)$ are $1/2 + \Delta$ and half of the entries are $1/2$. The multi-armed contextual bandit corresponding to $\bmu$ is given as follows. For each context $s \in \cS$, the mean reward is given by $r_{\bmu}(\cdot, s) = \bmu(s) \oplus\{1/2 - \alpha\}$, that is, the expected reward of first $2A$ arms are given by $\bmu(s)$ and those on last arm is given by $1/2-\alpha$, where $\alpha$ will be specified later. The reference policy is given by 
\begin{align*}
    \piref(a|s)=\frac{1}{2AC} \text{ if } a\in[2A] \text{ else } \piref(a|s)=\frac{C-1}{C}.
\end{align*} 

Given a context $s$ and instance $\bmu$, we use $a^*(\bmu(s))$ and $\bmu^*(s)$ (interchangeably) to denote all the set of optimal arms of $\bmu$ under context $s$. 
Given a context $s$, we use the notation $\bmu(s) \sim_{i,j} \blambda(s)$ if $a^*(\bmu(s)) \triangle a^*(\blambda(s)) = \{i,j\}$ with $i \in a^*(\bmu(s))$ and $j \in a^*(\blambda(s))$. We use $\mu(s)_i$ to denote the $i$-th entry of $\bmu(s)$, and when no ambiguity arises we also use $\mu(s)_i$ to denote the distribution with mean $\mu(s)_i$. Furthermore, we use the notation $\bmu \sim_{s, (i,j)} \blambda$ if $\bmu(s) \sim_{i,j} \blambda(s)$ and $\bmu(s') = \blambda(s')$ for all $s' \neq s$. 
Given an algorithm and a dataset, suppose the final output policy is $\pihat \in \cP(\cA | \cS)$. We use $\hat{a}$ to denote a random variable sampled from $\pihat(\cdot|s)$, that is, $\hat{a} \sim \pihat(\cdot|s)$, for a given $s$. We use $\pihat \sim \bmu$ if $\pihat$ is obtained by running some given algorithm on instance $\bmu$. Let $\alpha = \eta^{-1}\log(C-1) + \eta^{-1}\log2$. Then, for each context, the optimal policy on one optimum is given by
\begin{align*}
    \pi^*(a^*) = \frac{e^{\eta\Delta} / (2CA)}{(2C)^{-1}e^{\eta\Delta} + (2C)^{-1} + (C-1)C^{-1}e^{-\eta\alpha}} = A^{-1}\frac{e^{\eta\Delta}}{e^{\eta\Delta} + 2},
\end{align*}
which gives that $C^{\pi^*} \leq 2 C$.

We consider two instances such that $\bmu \sim_{s,(i,j)} \blambda $. We show that if $\PP_{\bmu}[\hat{a} = i]$ turns out to be large, then $\PP_{\blambda}[\hat{a} = i]$, on which $\pihat$ makes a mistake, is also somewhat large. In particular, by Proposition~\ref{prop:change-of-measure}, for any event $\cE$ and any $\gamma \in \RR$, 
\begin{align}
\PP_{\blambda}(\cE) \geq e^{-\gamma} \bigg[\PP_{\bmu}(\cE) - \PP_{\bmu}\bigg(\log \frac{\ud\PP_{\bmu}}{\ud\PP_{\blambda}}  > \gamma \bigg)\bigg]. \label{eq:slow-rate-COM-template}
\end{align}
Since the noise is standard Gaussian, given a dataset in which the number of occurrences of the $k$-th arm under context $x$ is $N_{n,x,k}$ and the corresponding empirical mean reward is $\hat{r}_{x,k}$, the log-likelihood ratio is given by
\begin{align*}
    \log \frac{\ud\PP_{\bmu}}{\ud\PP_{\blambda}} = \sum_{x \in \cS}\sum_{k \in \cA} \frac{N_{n,x,k}}{2}\big(\mu(x)_k - \lambda(x)_k\big)^2 + \sum_{x \in \cS}\sum_{k \in \cA}N_{n,x,k}\big( \mu(x)_k  - \lambda(x)_k \big)\big(\hat{r}_{x,k} - \mu(x)_k\big).
\end{align*}
Since $\bmu$ and $\blambda$ only differ on arms $i$ and $j$ under context $s$, with $|\mu(s)_k - \lambda(s)_k| = \Delta$ for $k \in \{i,j\}$, the first term equals $\Delta^2 T/2$, where $T \coloneqq N_{n,s,i} + N_{n,s,j}$, and the second term is a sum of $T$ independent centered Gaussian random variables with variance $\Delta^2$, conditioned on the actions. Each sample falls into $\{(s,i),(s,j)\}$ independently with probability $2 \cdot S^{-1}(2AC)^{-1} = (SCA)^{-1}$, so $\EE_{\bmu}[T] = n(SCA)^{-1}$. Applying Freedman's inequality (Lemma~\ref{lem:freedman}) with $t = 3n(SCA)^{-1}$ gives
\begin{align*}
    \PP_{\bmu}\big[T > 4n(SCA)^{-1}\big] \leq \exp \bigg(- \frac{9n^2 (SCA)^{-2}}{2n (SCA)^{-1} + 2n (SCA)^{-1}}\bigg) \leq \exp \bigg(- \frac{n}{2SCA}\bigg).
\end{align*}
Let $\cE_{T,n} \coloneq \{T \leq 4n(SCA)^{-1}\}$. On $\cE_{T,n}$, the first term is at most $2n\Delta^2(SCA)^{-1}$, and the conditional variance of the second term is at most $4n\Delta^2(SCA)^{-1}$. Applying Hoeffding's inequality (Lemma~\ref{lem:Hoeffding}) conditioned on the actions, we obtain
\begin{align}
    \PP_{\bmu}\bigg[\log \frac{\ud\PP_{\bmu}}{\ud\PP_{\blambda}}  > \frac{2n\Delta^2}{SCA} + \beta \bigg] & \leq \PP_{\bmu}\big[\cE_{T,n}^{\mathsf{C}}\big] + \exp\bigg(-\frac{SCA\beta^2}{8n\Delta^2}\bigg) \notag \\
    & \leq \exp \bigg(- \frac{n}{2SCA}\bigg) + \exp\bigg(-\frac{SCA\beta^2}{8n\Delta^2}\bigg). \label{eq:hoeffding}
\end{align}
Setting $\gamma = 2n\Delta^2(SCA)^{-1} + \beta$ in~\eqref{eq:slow-rate-COM-template} and using~\eqref{eq:hoeffding}, we obtain
\begin{align}
    \PP_{\blambda}(\cE) \geq \exp \bigg(- \frac{2n\Delta^2}{SCA} - \beta \bigg) \bigg[\PP_{\bmu}(\cE) -\exp \bigg(- \frac{n}{2SCA}\bigg) - \exp\bigg(-\frac{SCA\beta^2}{8n\Delta^2}\bigg) \bigg]. \label{eq:slow-rate-COM-final-template}
\end{align}
Now we set $\Delta = \sqrt{SAC(n\log A)^{-1}}$ and $\beta = 4\sqrt{2}$. Then $\exp(-SCA\beta^2/(8n\Delta^2)) = \exp(-4\log A) = A^{-4}$, and since $n \geq 10SAC\log A$, we have $\exp(-n/(2SCA)) \leq A^{-5} \leq A^{-4}$. Moreover, since $A \geq 3$,
\begin{align*}
    \kappa \coloneqq \exp \bigg(- \frac{2n\Delta^2}{SCA} - \beta \bigg) = \exp \bigg(- \frac{2}{\log A} - 4\sqrt{2} \bigg) \geq e^{-8}.
\end{align*}
Note that this choice of $\Delta$ also satisfies $\Delta \leq (\sqrt{10}\log A)^{-1} \leq 1/2$, so that all mean rewards lie in $[0,1]$. Consequently, \eqref{eq:slow-rate-COM-final-template} becomes
\begin{align}
    \PP_{\blambda}(\cE) \geq \kappa\big[\PP_{\bmu}(\cE) - 2A^{-4}\big]. \label{eq:slow-rate-pair}
\end{align}

Let $\cE_{s,i} = \{\hat{a} = i: \hat{a} \sim \pihat(\cdot|s)\}$, and let $\cV \coloneqq \binom{[2A]}{A}^{S}$ denote the set of all instances, so that $|\cV| = \binom{2A}{A}^S$. For any $s \in \cS, i \in [2A]$, define the bipartite graph $G_{s,i}$ with left vertices $L_{s,i} = \{\bmu \in \cV: i \in \bmu(s)\}$ and right vertices $R_{s, i} = \{\blambda \in \cV: i \notin \blambda(s)\}$; and connect $(\bmu, \blambda) \in L_{s,i} \times R_{s, i}$ if $\exists j \in [2A] \backslash\{i\}$ such that $\bmu \sim_{s, (i, j)} \blambda$. Then by construction, each $\bmu \in L_{s,i}$ and $\blambda \in R_{s,i}$ has exactly $A$ neighbors, i.e., $G_{s, i}$ is an $A$-regular bipartite graph, which admits a perfect matching $M_{s,i} \subset L_{s,i} \times R_{s,i}$ by Lemma~\ref{lem:perfect-matching}.\footnote{We suppress the last arm with reward $1/2-\alpha$ in this argument to avoid notation clutter. Hence, for each $G_{s,i}$, $|L_{s,i}| = |R_{s,i}| = |\cV|/2$.} Any pair $(\bmu, \blambda) \in M_{s,i}$ fits~\eqref{eq:slow-rate-pair}, and every element of $L_{s,i}$ and $R_{s,i}$ appears exactly once in $M_{s,i}$. Therefore, summing~\eqref{eq:slow-rate-pair} over $(\bmu,\blambda) \in M_{s,i}$, $s \in \cS$ and $i \in [2A]$, and interchanging the order of summation yields
\begin{align*}
    \sum_{\blambda \in \cV}\sum_{s \in \cS}\sum_{i \in [2A]\backslash \blambda^*(s)} \PP_{\blambda}(\cE_{s,i}) \geq \kappa \bigg[\sum_{\bmu \in \cV}\sum_{s\in \cS}\sum_{i \in \bmu^*(s)}\PP_{\bmu}(\cE_{s,i}) - 2A^{-4} \cdot SA|\cV| \bigg],
\end{align*}
where we use $\sum_{s,i}|M_{s,i}| = S \cdot 2A \cdot |\cV|/2 = SA|\cV|$. For any policy $\pi$, we use the shorthand $\pi(\tilde{a}|s) \coloneqq \sum_{a \notin a^*} \pi(a|s)$ to denote the total probability assigned to suboptimal arms (including the last arm), and let $Y \coloneqq \sum_{\blambda \in \cV}\sum_{s \in \cS} \EE_{\blambda}[\pihat(\tilde a|s)]$. The left-hand side above is at most $Y$, and $\sum_{i \in \bmu^*(s)}\PP_{\bmu}(\cE_{s,i}) = 1 - \EE_{\bmu}[\pihat(\tilde a|s)]$. Hence,
\begin{align*}
    Y \geq \kappa\big[S|\cV| - Y - 2A^{-3}S|\cV|\big], \quad \text{i.e.,} \quad Y \geq \frac{\kappa}{1+\kappa}\big(1-2A^{-3}\big)S|\cV| \geq \frac{e^{-8}}{2}\cdot\frac{25}{27}\,S|\cV| \geq e^{-9}S|\cV|,
\end{align*}
where we use $\kappa \in [e^{-8}, 1]$ and $A \geq 3$. As a consequence, there exists some $\blambda \in \cV$ such that
\begin{align}
      \EE_{\pihat \sim \blambda} \EE_{s\sim \unif(\cS)} \big[\pihat(\tilde a|s)\big] = \frac{1}{S}\sum_{s\in \cS} \EE_{\blambda}\big[\pihat(\tilde a|s)\big] \geq e^{-9}. \label{eq:slowrate-pibad-upper}
\end{align}
Now, we consider the suboptimality gap for the instance $\blambda$ satisfying \eqref{eq:slowrate-pibad-upper}. Let $\pi^*_{\blambda}$ be the optimal policy under the instance $\blambda$, i.e., $\pi_{\blambda}^*(\cdot|s) \propto \piref(\cdot|s) \exp\big(\eta r_{\blambda}(\cdot,s)\big)$. For any context $s \in \cS$, let $a_{\blambda}^*(s)$ be the set of optimal arms for the instance $\blambda$. We will omit the subscript $\blambda$ when it will not cause any confusion. Recall that $\hat{\pi}$ is the policy produced by the algorithm after interacting with the problem instance. We write $\hat{\pi} \sim \blambda$ to indicate that the policy is obtained from interaction with instance $\blambda$. Thus, for any given algorithm, the expected suboptimality gap under instance $\blambda$ can be written as $\EE_{\pihat \sim \blambda}\big[\subopt(\pihat, \pi_{\blambda}^*)\big]$. Using Lemma \ref{lem:kl-subopt-decompose} and the data processing inequality~\citep[Theorem~2.17]{polyanskiy2025information},
we know that
\begin{align*}
    \eta\EE_{\pihat \sim \blambda}\big[\subopt(\pihat, \pi_{\blambda}^*)\big] &= \EE_{\hat \pi \sim \blambda}\big[\KL(\hat \pi \|\pi^*_{\blambda})\big]\\
    &\ge \EE_{\hat \pi \sim \blambda}\EE_{s \sim \unif(\cS)}\bigg[\hat \pi(\tilde a|s)\log \frac{\hat\pi(\tilde a|s)}{\pi^*_{\blambda}(\tilde a|s)} + \hat \pi(a^*|s) \log \frac{\hat \pi(a^*|s)}{\pi^*_{\blambda}(a^*|s)}\bigg].
\end{align*}
where we recall the shorthand $\pi(\tilde{a}|s) = \sum_{a \notin a^*} \pi(a|s)$ and similarly write $\pi(a^*|s) \coloneqq \sum_{a \in a^*} \pi(a|s)$ for the total probability assigned to optimal arms. Note that for any $s \in \cS$, we have $\pi^*_{\blambda}(\tilde a|s) = 2(2 + e^{\eta\Delta})^{-1}$ and $\pi^*_{\blambda}(a^*|s) = e^{\eta\Delta}(2 + e^{\eta\Delta})^{-1}$, which are independent of $s$. Hence, we drop the dependence on $s$ and write $\pi_{\blambda}^*(\tilde a)$, $\pi_{\blambda}^*(a^*)$ without ambiguity. Applying Jensen's inequality to the convex function $f(x) = x \log (x/a)$, where $a$ is a constant, we have
\begin{align}
    \eta \EE_{\pihat \sim \blambda}\big[\subopt(\pihat, \pi_{\blambda}^*)\big] 
    & \geq \EE_{\pihat \sim \blambda} \EE_{s \sim \unif(\cS)}[\pihat(\tilde a|s)] \log \frac{ \EE_{\pihat \sim \blambda} \EE_{s \sim \unif(\cS)}[\pihat(\tilde a|s)]}{\pi^*_{\blambda}(\tilde a)} \notag \\
    & \quad + \EE_{\pihat \sim \blambda} \EE_{s \sim \unif(\cS)}[\pihat(a^*|s)] \log \frac{\EE_{\pihat \sim \blambda} \EE_{s \sim \unif(\cS)}[\pihat(a^*|s)]}{\pi^*_{\blambda}(a^*)} \label{eq:slowrate-subopt}.
\end{align}
Therefore, the suboptimality gap can be lower bounded by the KL divergence between two Bernoulli distributions with means $\EE_{\pihat \sim \blambda} \EE_{s \sim \unif(\cS)}[\pihat(\tilde a|s)]$ and $\pi_{\blambda}^*(\tilde a)$. Our next step is to show that $\EE_{\pihat \sim \blambda} \EE_{s \sim \unif(\cS)}[\pihat(\tilde a|s)]$ is larger than $\pi^*_{\blambda}(\tilde a) = 2(2 + e^{\eta\Delta})^{-1}$. Since $n \leq \eta^2SAC/(576 \log A)$, we have
\begin{align*}
    \eta\Delta = \eta \sqrt{\frac{SAC}{n \log A}} \geq 24.
\end{align*}
Therefore, we know that 
\begin{align*}
    \pi^*_{\blambda}(\tilde a) = \frac{2}{2 + e^{\eta\Delta}} \leq 2\exp(-24) \leq \exp(-9) \stackrel{\eqref{eq:slowrate-pibad-upper}}{\leq} \EE_{\pihat \sim \blambda} \EE_{s \sim \unif(\cS)}[\pihat(\tilde a|s)].
\end{align*}
As a result, we can plug~\eqref{eq:slowrate-pibad-upper} into \Cref{eq:slowrate-subopt} and compute as follows,
\begin{align}
    & \eta\EE_{\pihat \sim \blambda}\big[\subopt(\pihat, \pi^*_{\blambda})\big]  \notag\\
    &\quad \geq \mathsf{KL}\Big( \mathrm{Bern}\big( \EE_{\pihat \sim \blambda} \EE_{s \sim \unif(\cS)}[\pihat(\tilde a|s)] \big) \Bigm\| \mathrm{Bern}\big( \pi^*_{\blambda}(\tilde a) \big) \Big) \notag\\
    &\quad \geq \mathsf{KL}\Big( \mathrm{Bern}\big( e^{-9} \big) \Bigm\| \mathrm{Bern}\big( \pi^*_{\blambda}(\tilde a) \big) \Big) \notag\\
    & \quad \geq  \underbrace{e^{-9} \log \frac{e^{-9}}{2(2 + \exp(\eta\Delta))^{-1}}}_{I_1} + \underbrace{\log \frac{1 - e^{-9} }{ 1- 2(2 + \exp(\eta\Delta))^{-1}} }_{I_2}, \label{eq:slow-rate-decomp}
\end{align}
where the second inequality holds due to $\kl{\mathrm{Bern}(p_1)}{\mathrm{Bern}(q)} \geq \kl{\mathrm{Bern}(p_2)}{\mathrm{Bern}(q)}$ whenever $1 > p_1 \geq p_2 \geq q > 0$, and the third inequality holds since $1 - e^{-9} \leq 1$ and $1-e^{-9} \leq 1 - 2(2+\exp(\eta\Delta))^{-1}$, so that the second term of the KL divergence is negative and multiplying it by $1-e^{-9} \leq 1$ only increases it.
For the first term $I_1$, we have
\begin{align}
    I_1 \geq e^{-9} \log \frac{e^{-9}\exp(\eta\Delta)}{2} = e^{-9} \big(\eta\Delta - \log2 - 9\big). \label{eq:regularized-slow-rate-I-1}
\end{align}
For the second term $I_2$, we have
\begin{align}
    I_2 & = - \log \bigg(1 + \frac{e^{-9} - 2(2 + \exp(\eta\Delta))^{-1}}{1 - e^{-9}}\bigg)  \geq - \frac{e^{-9} - 2(2 + \exp(\eta\Delta))^{-1}}{1 - e^{-9}} \geq - 2e^{-9}, \label{eq:regularized-slow-rate-I-2}
\end{align}
where the first inequality holds due to $\log(1+x) \leq x$, and the last inequality is by $1 - e^{-9} \geq 1/2$ and $2(2 + \exp(\eta\Delta))^{-1} \geq 0$. Substituting \Cref{eq:regularized-slow-rate-I-1,eq:regularized-slow-rate-I-2} into \Cref{eq:slow-rate-decomp}, and using $\eta\Delta/2 \geq 12 \geq 11 + \log 2$, yields
\begin{align*}
    \EE_{\pihat \sim \blambda}\big[\subopt(\pihat, \pi^*_{\blambda})\big]  & \geq \eta^{-1} e^{-9}\cdot \big(\eta \Delta - 11 - \log2 \big) \geq \eta^{-1} e^{-9}\cdot \frac{\eta\Delta}{2} \gtrsim \sqrt{\frac{SAC}{n\log A}},
\end{align*}
which finishes the proof of Case 1. Since $A \geq |\cA|/3$, $\log A \leq \log |\cA|$ and $C = C^*/2$, this bound is $\Omega\big(\sqrt{S|\cA|C^*/(n\log|\cA|)}\big)$.

\noindent\textbf{Case 2: $3 \leq |\cA| \leq 6$.} In this case, the action-space size is bounded by a universal constant, and the desired bound follows from the two-armed construction of~\citet[Theorem~2.11]{zhao2026towards}, which establishes an $\Omega(\sqrt{SC^*n^{-1}})$ lower bound in the small regularization regime $\eta^2 = \Omega(n(SC^*)^{-1})$. Following the same embedding as in Case~2 of the proof of Theorem~\ref{thm:lowerbound-fast}, i.e., appending $|\cA|-2$ actions with reference probability $0$ and reward $0$, the embedded $|\cA|$-armed instances are equivalent to the original two-armed instances, and the proof of~\citet[Theorem~2.11]{zhao2026towards} applies verbatim. Since $|\cA| \leq 6$, we have $\Omega(\sqrt{SC^*n^{-1}}) = \Omega(\sqrt{S|\cA|C^*n^{-1}})$, and the regime conditions coincide up to absolute constants.

Combining the two cases establishes Theorem~\ref{thm:lowerbound-slow}.
\end{proof}

\begin{remark}[The Remaining Range $n = \tilde{O}(SAC^*)$]\label{rmk:small-n}
Theorem~\ref{thm:lowerbound-slow} requires $n \geq 10SAC^*\log A$. We now show that in the remaining range $n < 10SAC^*\log A$, the minimax risk~\eqref{eq:minimax-risk} is $\Omega((\log A)^{-1})$ and $O(1)$, which is tight up to logarithmic factors.

For the upper bound, we show that, for any reward function $r: \cS \times \cA \to [0,1]$, the induced policy $\pi_r(a|s) \propto \piref(a|s)\exp(\eta r(s,a))$ satisfies $\subopt(\pi_r) \leq 2 = O(1)$. Indeed, $J(\pi^*) \leq \max_{\pi}\EE_{(s,a)\sim\rho\times\pi}[\fgt(s,a)] \leq 1$ since the KL divergence is nonnegative, and by Jensen's inequality, $\kl{\pi_r(\cdot|s)}{\piref(\cdot|s)} = \eta\EE_{a \sim \pi_r(\cdot|s)}[r(s,a)] - \log\EE_{a\sim\piref(\cdot|s)}[\exp(\eta r(s,a))] \leq \eta$, which gives $J(\pi_r) \geq -1$. In particular, choosing $r \equiv 0$ yields $\pi_r = \piref$ and $\subopt(\piref) \leq 1$. Therefore, even an algorithm that ignores the data, selects an arbitrary reward function $r: \cS\times \cA \to [0,1]$, and outputs $\pi_r$ guarantees an $O(1)$ suboptimality gap.

For the lower bound, let $m = \lceil 10SAC^*\log A \rceil$ and suppose that $\eta$ is large enough such that $m \leq \eta^2SAC^*/(3456\log A)$. For any algorithm $\mathtt{Alg}$ that takes $n < m$ samples, we construct an algorithm $\mathtt{Alg}'$ that receives $m$ samples, discards the last $m-n$ samples, and applies $\mathtt{Alg}$ on the first $n$ samples. Since $\mathtt{Alg}'$ and $\mathtt{Alg}$ output policies with the same distribution, and Theorem~\ref{thm:lowerbound-slow} applies to $\mathtt{Alg}'$, there exists an instance in $\mathrm{MAB}(C^*)$ on which $\mathtt{Alg}'$, and hence $\mathtt{Alg}$, incurs an expected suboptimality gap of at least
\begin{align*}
    \Omega\bigg(\sqrt{\frac{SAC^*}{m\log A}}\bigg) = \Omega\big((\log A)^{-1}\big),
\end{align*}
which gives the $\Omega((\log A)^{-1})$ lower bound.
\end{remark}

\section{Proof of Theorems in Appendix~\ref{app:multi-armed-multiple}}

\subsection{Proof of Theorem~\ref{thm:multi-armed-lower}}\label{app:proof-multi-armed-lower}

\begin{proof}[Proof of Theorem~\ref{thm:multi-armed-lower}]
We consider the multi-armed bandit class parameterized by mean vectors:
\begin{align*}
\bmu \in \cV_K:= \bigg\{ \bmu \in \{\pm 1\}^A \bigg| \sum_{i=1}^A \ind(\bmu_i=-1) = K \bigg\},
\end{align*}
i.e., exactly $K$ entries of $\bmu$ are $-1$. Unless stated otherwise, we assume that noise of rewards is standard Gaussian. For any $\bmu \in \cV_K$, the corresponding instance is given by $([A], r_{\bmu}, \unif(\cA))$ where the reward is given by $r_{\bmu}(i) = 1/2 + \bmu_i \Delta$, with $\Delta > 0$ to be specified later.
For any $\bmu$, we denote by $a^*(\bmu)$ the set of optimal actions (those with reward $1/2+\Delta$). For two instances $\bmu$ and $\blambda$, we write $\bmu \sim_{i,j} \blambda$ if $a^*(\bmu) \triangle a^*(\blambda) = \{i,j\}$ with $i \in a^*(\bmu)$ and $j \in a^*(\blambda)$, i.e., instances $\bmu$ and $\blambda$ differ only in that action $i$ is optimal under $\bmu$ while action $j$ is optimal under $\blambda$. When there is no ambiguity, we also write $\bmu_i$ for the distribution with mean $1/2 + \bmu_i \Delta$ and standard Gaussian noise.
Given an algorithm and a dataset, let $\pihat$ denote the output policy. We write $\hat{a} \sim \pihat$ for a random action sampled from $\pihat$, and $\pihat \sim \bmu$ to indicate that $\pihat$ is produced by running the algorithm on instance $\bmu$.

Consider two instances that $\bmu \sim_{i,j} \blambda $. We show that if $\PP_{\bmu}[\pihat = i]$ is large then $\PP_{\blambda}[\pihat = i]$ is also large, resulting in an error on $\blambda$. By Proposition~\ref{prop:change-of-measure}, we know that for any event $\cE$,
\begin{align}
    \PP_{\blambda}(\cE) \geq e^{-\gamma} \Bigg[\PP_{\bmu}(\cE) - \PP_{\bmu}\bigg[\log \frac{\ud\PP_{\bmu}}{\ud\PP_{\blambda}}  > \gamma \bigg]\Bigg]. \label{eq:multi-armed-change-of-measure}
\end{align}
We first compute the term $\log \ud\PP_{\bmu}/\ud\PP_{\blambda}$. Specifically, consider a dataset $\cD = (a_1,r_1,\ldots, a_n,r_n)$. Let $p_{\bmu}(\cD)$ denote the joint density of $\cD$ under the bandit instance $\bmu$, and define $p_{\blambda}(\cD)$ analogously. Then,
\begin{align*}
    \frac{\ud\PP_{\bmu}}{\ud\PP_{\blambda}} &= \frac{p_{\bmu}(\cD)}{p_{\blambda}(\cD)}= \prod_{i=1}^n \frac{p_{\bmu}(a_i,r_i)}{p_{\blambda}(a_i,r_i)},
\end{align*}
where the last equation holds because the samples in $\cD$ are independent and identically distributed. Moreover, using Bayes' rule and the fact that the reward distribution is standard Gaussian, we have
\begin{align*}
    \log \prod_{i=1}^n \frac{p_{\bmu}(a_i,r_i)}{p_{\blambda}(a_i,r_i)} &= \log \prod_{i=1}^n \frac{p_{\bmu}(r_i|a_i)}{p_{\blambda}(r_i|a_i)}\\
    &= \sum_{k \in \cA}\sum_{i=1}^n \ind(a_i=k) \log \frac{p_{\bmu}(r_i|a_i)}{p_{\blambda}(r_i|a_i)}\\
    &= \sum_{k \in \cA}\sum_{i=1}^n \ind(a_i=k) \log \frac{\exp[-(r_i-\mu_k)^2/2]}{\exp[-(r_i-\lambda_k)^2/2]}\\
    &= \frac{1}{2} \sum_{k \in \cA}\sum_{i=1}^n \ind(a_i=k) \Big[(r_i-\lambda_k)^2 - (r_i-\mu_k)^2\Big]\\
    &= \frac{1}{2} \sum_{k \in \cA}\sum_{i=1}^n \ind(a_i=k) \Big[2(r_i-\mu_k)(\mu_k-\lambda_k) +  (\mu_k-\lambda_k)^2\Big].
\end{align*}
Therefore, the logarithmic Radon–Nikodym derivative can be represented as
\begin{align*}
    \log \frac{\ud\PP_{\bmu}}{\ud\PP_{\blambda}} &= \frac{1}{2} \sum_{k \in \cA}\sum_{i=1}^n \ind(a_i=k) \Big[2(r_i-\mu_k)(\mu_k-\lambda_k) +  (\mu_k-\lambda_k)^2\Big]\\
    &= \underbrace{\frac{1}{2} \sum_{k \in \cA}N(k)(\mu_k-\lambda_k)^2}_{I_1} + \underbrace{\sum_{k \in \cA}N(k)( \mu_k  - \lambda_k )(\hat{r}_{k} - \mu_k)}_{I_2},
\end{align*}
where $N(k) = \sum_{i=1}^n \ind(a_i=k)$
denotes the number of occurrences of arm $k$, and $\hat r_k = \sum_{i=1}^n \ind(a_i=k) r_i/N(k)$ denotes the empirical mean reward of arm $k$. 
For $I_1$, we first recall our choice of $\bmu \sim_{i,j} \blambda$, such that the two instances only differ in two arms, with $|\mu_k-\lambda_k| = 2\Delta$ if and only if $k \in \{i,j\}$. Therefore, we have
\begin{align*}
    I_1 = 2\Delta^2 \big[N(i) +N(j)\big].
\end{align*}
Using standard concentration inequalities (Lemma~\ref{lem:Chernoff}), we show that $N(i) + N(j) = O(n/A)$ with high probability. 
Specifically, by Lemma~\ref{lem:Chernoff}, we have
\begin{align*}
    \PP[N(i) \ge 2n/A] &\le \exp\Big(- \frac{n}{3A}\Big),\\
    \PP[N(j) \ge 2n/A] &\le \exp\Big(- \frac{n}{3A}\Big),
\end{align*}
Then, taking a union bound, we have
\begin{align*}
    \PP[N(i) + N(j) \ge 4n/A] \leq 2\exp \Big(-\frac{n}{3A}\Big).
\end{align*}
Let $\cE_{i,j;n}:=\{N(i) + N(j) \leq 4n/A\}$. We have seen $\PP[\cE_{i,j;n}^{\mathsf{C}}] \le 2\exp(-n/[3A])$. Moreover, 
\begin{align}
\label{eq:0001}
    \PP[I_1 \ge 8n\Delta^2/A] \le 2\exp \Big(-\frac{n}{3A}\Big).
\end{align}
For $I_2$, we have
\begin{align*}
    I_2 &= \sum_{k \in \cA}N(k)( \mu_k  - \lambda_k )(\hat{r}_{k} - \mu_k)\\
    &= 2 \Delta \sum_{s=1}^n \big[ \ind(a_s=i)(r_s-\mu_i) - \ind(a_s=j)(r_s - \mu_j)\big],
\end{align*}
where the last term is the summation of $[N(i)+N(j)]$ independent Gaussians. Note that the selected actions and the noise of rewards are independent. We can use Hoeffding's inequality (Lemma \ref{lem:Hoeffding}), conditioned on the event $\cE_{i,j;n}$. Thus, we have
\begin{align*}
    \PP_{\bmu}\Big[\big\{I_2 \ge 2\beta \Delta\big\} \bigcap \cE_{i,j;n} \Big] \le \exp\Big(-\frac{A\beta^2}{8n}\Big).
\end{align*}
Using the union bound, we have
\begin{align}
\notag
    \PP_{\bmu}\big[I_2 \ge 2\beta \Delta\big] &\le \PP_{\bmu}[\cE_{i,j;n}^{\mathsf{C}}] + \PP_{\bmu}\Big[\big\{I_2 \ge 2\beta \Delta\big\} \bigcap \cE_{i,j;n} \Big]\\\label{eq:0002}
    &\le 2\exp \Big(-\frac{n}{3A}\Big) + \exp\Big(-\frac{A\beta^2}{8n}\Big).
\end{align}
Combining~\eqref{eq:0001} and~\eqref{eq:0002}, we get
\begin{align*}
    \PP_{\bmu}\bigg[\log \frac{\ud\PP_{\bmu}}{\ud\PP_{\blambda}}  > \frac{8n\Delta^2}{A} + 2\beta\Delta \bigg] & \le \PP_{\bmu}\bigg[I_1 \geq \frac{8n\Delta^2}{A}\bigg] + \PP_{\bmu}[I_2 \geq 2\beta\Delta]\\
    &\leq 4\exp \bigg(- \frac{n}{3A}\bigg) +  \exp\bigg(-\frac{A\beta^2}{8n}\bigg) \\
    & \leq 5\exp\bigg(-\frac{A\beta^2}{8n}\bigg),
\end{align*}
where the last inequality holds \emph{under the premise of} $n \ge \beta A$, which is satisfied by the choice of $\beta$ below whenever $n \geq 24A\log A$. Now we plug in the above inequality back to the change-of-measure argument~\eqref{eq:multi-armed-change-of-measure} with $\gamma = 8n\Delta^2/A + 2\beta \Delta$. Then, we obtain that
\begin{align*}
    \PP_{\blambda}(\cE) &\geq \exp\big(-(8n\Delta^2/A + 2\beta \Delta)\big)\bigg[\PP_{\bmu}(\cE) - 5\exp\bigg(-\frac{A \beta^2 }{8n}\bigg)\bigg] .
\end{align*}
Let $\cE_i = \{\hat{a} = i: \hat{a} \sim \pihat(\cD)\}$, we traverse over every $\bmu$, $\blambda$, $i$ and $j$ such that $\bmu \sim_{i,j} \blambda$ and sum up the inequality above. For each $(\blambda, i \notin a^*(\blambda))$, there exists $(A-K)$ $j$ such that the reward on $i$ and $j$ to be switched to get $\bmu$ such that $\bmu \sim_{i,j} \blambda$. Therefore, each $(\blambda, i \notin a^*(\blambda))$ appears $(A-K)$ times. Similarly, each $(\bmu, i\in a^*(\bmu))$ appears $K$ times and we have $\binom{A}{K}(A-K)K$ inequalities in total. These give the multiplicative factors in front of each term:
\begin{align*}
    & (A-K)\sum_{\blambda}\sum_{i \notin a^*(\blambda)} \PP_{\blambda}(\cE_{i}) \\
    & \quad  \geq \exp\bigg(-\frac{8n\Delta^2}{A} - 2\beta \Delta\bigg) \bigg[K\sum_{\bmu}\sum_{i \in a^*(\bmu)}\PP_{\bmu}(\cE_{i}) - 5K(A-K)\binom{A}{K}\exp\bigg(-\frac{A\beta^2}{8n}\bigg) \bigg] \\
    & \quad \geq \exp\bigg(-\frac{8n\Delta^2}{A} - 2\beta\Delta\bigg) K \bigg[\sum_{\blambda}\bigg(1 -\sum_{i \notin a^*(\blambda)} \PP_{\blambda}(\cE_{i})\bigg) - 5(A-K)\binom{A}{K}\exp\bigg(-\frac{A\beta^2}{8n}\bigg) \bigg],
\end{align*}
which implies that
\begin{align*}
    & \sum_{\blambda}\sum_{i \notin a^*(\blambda)} \PP_{\blambda}(\cE_{i}) \\
    & \quad  \geq \frac{K\exp\big(-8n\Delta^2/A - 2\beta\Delta\big)}{A-K + K\exp\big(-8n\Delta^2/A - 2\beta\Delta\big)} \binom{A}{K} \bigg[1 - 5(A-K)\exp\bigg(-\frac{A\beta^2}{8n}\bigg)\bigg] \\
    & \quad \geq \exp\bigg(-\frac{8n\Delta^2}{A} - 2\beta\Delta\bigg)\frac{K}{A} \binom{A}{K} \bigg[1 - 5(A-K)\exp\bigg(-\frac{A\beta^2}{8n}\bigg)\bigg],
\end{align*}
where the second inequality holds due to $\exp\big(-8n\Delta^2/A - 2\beta\Delta\big) \leq 1$. Now we select $\beta = 2\sqrt{6n\log A/A}$ and $\Delta  = \sqrt{A(n\log A)^{-1}}$, we have 
\begin{align*}
    & \exp\bigg(-\frac{8n\Delta^2}{A} - 2\beta\Delta\bigg)= \exp\bigg(-\frac{8}{\log A} - 4\sqrt{6}\bigg) \geq e^{-22}, \\
    & \exp\bigg(-\frac{A\beta^2}{8n}\bigg) = \exp\bigg(-3 \log A\bigg) = A^{-3}.
\end{align*}
Consequently, since $5(A-K)A^{-3} \leq 5(A-1)A^{-3} \leq 5/8$ for all $A \geq 2$, we get
\begin{align*}
    \sum_{\blambda}\sum_{i \notin a^*(\blambda)} \PP_{\blambda}(\cE_{i}) \gtrsim \frac{K}{A}\binom{A}{K}(1 - 5(A-K)A^{-3}) \gtrsim \frac{K}{A}\binom{A}{K}. 
\end{align*}
Therefore, we know that there exists an $\blambda \in \cV_K$ such that
\begin{align*}
    \sum_{i \notin a^*(\blambda)} \PP_{\blambda}(\cE_{i}) \gtrsim \frac{K}{A}.
\end{align*}
Recall that $\EE_{\blambda}[\subopt(\pihat)] = 2\Delta \sum_{i \notin a^*(\blambda)} \PP_{\blambda}(\cE_{i})$, we obtain that there exists an $\blambda \in \cV_K$ such that
\begin{align*}
    \EE_{\blambda}[\subopt(\pihat)] \gtrsim \frac{K\Delta}{A} \geq \frac{K}{\sqrt{n A \log A }} ,
\end{align*}
which finishes the proof.
\end{proof}

\subsection{Proof of Theorem~\ref{thm:multi-armed-upper}}\label{app:proof-multi-armed-upper}

\begin{proof}[Proof of Theorem~\ref{thm:multi-armed-upper}] Without loss of generality, we consider the optimal reward $r^*=0$ here. Then, by our bounded reward assumption, we have $r(a) \geq -1$ for all $a \in \cA$. We assume $\cA = [A]$ and $\cA^* = [A-K]$. We further denote the suboptimality gap on arm $k$ to be $2\Delta_k$, or $r(k)=-2\Delta_k$, where $1/2>\Delta_k>0$ for all $k > A-K$. Thus, the suboptimality gap of the output policy $\hat \pi$, which deterministically selects $\hat a$, is $2\Delta_{\hat a}$ and therefore the expected suboptimality gap is given by
\begin{align*}
    \EE[\subopt(\pi)] &= 2\sum_{k>A-K} \PP[\hat a = k] \Delta_k.
\end{align*}
For an auxiliary accuracy parameter $\epsilon > 0$, we first show that $\EE[\subopt(\pi)] \leq 2\epsilon$ as long as 
\begin{align*}
    n \geq 8A\log\bigg(\frac{8A}{\epsilon}\bigg) + \frac{64K^2\log A}{A\epsilon^2} = \tilde{O}(K^2/A\epsilon^2).
\end{align*}

We first show that if $\Delta_k = O(A\epsilon/K)$, then the suboptimality gap on these arms accumulates to at most an order of $\epsilon$, thanks to having sufficiently many optimal arms. In particular, since our samples $a_i$ and reward noise $\varepsilon_i$ are drawn independently and symmetrically with respect to $\cA$, we know that for all $i > A-K$ and $j \leq A-K$, the probability that $\fls(i)$ takes the maximum is no greater than the probability that $\fls(j)$ takes the maximum, since $r(i) < r(j)$. This implies that
\begin{align*}
    \sum_{i > A-K} \PP[\hat a = i] \leq \frac{K}{A}.
\end{align*}
Therefore, the total suboptimality gap on those suboptimal arm $k$ with $\Delta_k \leq A\epsilon/(4K)$ results in at most $\epsilon/2$ suboptimality gap.
\begin{align}
    \EE\big[\subopt(\pi)\big] &=  2\sum_{k>A-K} \PP[\hat a = k] \ind\{\Delta_k \leq A\epsilon/(4K)\}\Delta_k \notag\\
    & \quad + 2\sum_{k>A-K} \PP[\hat a = k] \ind\{\Delta_k > A\epsilon/(4K)\}\Delta_k \notag \\
    & \leq \frac{\epsilon}{2} + \underbrace{2\sum_{k>A-K} \PP[\hat a = k] \ind\{\Delta_k > A\epsilon/(4K)\}\Delta_k}_{X}, \label{eq:multi-armed-upper-decomposition}
\end{align}
Therefore, we only need to consider the expected suboptimality incurred by those arm where $\Delta_k = \Omega(\epsilon A/K)$, denoted by $X$ in~\eqref{eq:multi-armed-upper-decomposition}. For each of these arms, since the reward gap is large enough, the probability that one of them is selected as the maximum of empirical mean is small. Specifically, we consider a superset of the event $\hat a = k$ for a given $k$:
\begin{align*}
    \cE_k = \{ \fls(k) \geq -\Delta_k \} \cup \{ \fls(1) \leq -\Delta_k \}.
\end{align*}
It is easy to see that $\{\hat a = k\} \subseteq \cE_k$, since on the complement of $\cE_k$, $\fls(k) < -\Delta_k$, but on the first arm (which is optimal) we have $\fls(1) > -\Delta_k > \fls(k)$, so $k$ cannot be selected. Now, let $\cX_k = \{ \fls(k) \geq -\Delta_k \}$ and $\cU = \{ \fls(1) \leq -\Delta_k \}$. Then, $\PP[\hat a = k] \leq \PP(\cX_k \cup \cU) \leq \PP(\cX_k) +\PP(\cU)$. We first bound the probability of $\cU$. Since $N(1) \sim \mathsf{Bin}(n,1/A)$, by Lemma~\ref{lem:binary-concentration}, we have
\begin{align*}
    \PP\Big[N(1) \leq \frac{n}{2A}\Big] \leq \exp \Big( -\frac{n}{8A}\Big).
\end{align*}
Conditioning on $N(1) \geq n/(2A)$, by Lemma~\ref{lem:azuma-hoeffding}, we know that
\begin{align*}
    \PP[\fls(1) \leq -\Delta_k] \leq \exp \bigg(-\frac{N(1)\Delta_k^2}{2}\bigg) \leq \exp \bigg(-\frac{n\Delta_k^2}{4A}\bigg).
\end{align*}
Taking a union bound, we know that
\begin{align*}
    \PP(\cU) \leq \exp \bigg( -\frac{n}{8A}\bigg) + \exp\bigg(-\frac{n\Delta_k^2}{4A}\bigg).
\end{align*}
Noticing that $\cX_k$ and $\cU$ are symmetric, we can repeat the above argument to obtain an upper bound of $\PP[\cX_k]$:
\begin{align*}
    \PP(\cX_k) \leq \exp \bigg( -\frac{n}{8A}\bigg) + \exp\bigg(-\frac{n\Delta_k^2}{4A}\bigg).
\end{align*}
Taking a union bound over $\cX_k$ and $\cU$, we obtain that
\begin{align*}
    \PP(\cE_k) \leq 2\exp \bigg( -\frac{n}{8A}\bigg) + 2\exp\bigg(-\frac{n\Delta_k^2}{4A}\bigg).
\end{align*}
Now we are ready to bound the second term $X$ in~\eqref{eq:multi-armed-upper-decomposition}. Specifically, we have
\begin{align*}
    X & = 2\sum_{k>A-K} \PP[\hat a = k] \ind\{\Delta_k > A\epsilon/(4K)\}\Delta_k \\
    & \leq 4\sum_{k>A-K} \ind\{\Delta_k > A\epsilon/(4K)\} \bigg[ \exp \bigg( -\frac{n}{8A}\bigg) + \exp\bigg(-\frac{n\Delta_k^2}{4A}\bigg) \bigg]\Delta_k \\
    & = \underbrace{4\sum_{k>A-K} \ind\{\Delta_k > A\epsilon/(4K)\} \exp \bigg( -\frac{n}{8A}\bigg) \Delta_k}_{(\text{i})} \\
    & \quad + \underbrace{4\sum_{k>A-K} \ind\{\Delta_k > A\epsilon/(4K)\}\exp\bigg(-\frac{n\Delta_k^2}{4A}\bigg) \Delta_k}_{(\text{ii})}.
\end{align*}
For the first term $(\text{i})$, our selection of $n$ ensures $n \geq 8A\log(8A/\epsilon)$, resulting in $\exp(-n/[8A]) \leq \epsilon/(8K)$. Consequently, $(\text{i})$ can be bounded as
\begin{align}
    (\text{i}) \leq 4\sum_{k > A-K} \Delta_k \frac{\epsilon}{8K} \leq \frac{\epsilon}{4}, \label{eq:multi-armed-upper-x1}
\end{align}
where the last inequality holds due to $\Delta_k \leq 1/2$. For the second term $(\text{ii})$, we rewrite $\Delta_k = C_k A\epsilon/(4K)$, where $C_k \geq 1$. For fixed $k$, we have
\begin{align*}
    \exp\bigg(-\frac{n\Delta_k^2}{4A}\bigg) \Delta_k = \exp \bigg(-\frac{nAC_k^2\epsilon^2}{64K^2}\bigg) \frac{C_k A\epsilon}{4K}.
\end{align*}
By our choice of $n$, we have $n \geq 64K^2\log A/(A\epsilon^2)$, which gives
\begin{align*}
    \exp \bigg(-\frac{nAC_k^2\epsilon^2}{64K^2}\bigg) \frac{C_k A\epsilon}{4K} \leq \exp \big(- C_k^2 \log A\big)\frac{C_k A\epsilon}{4K} = \frac{\epsilon}{4K} C_k A^{1-C_k^2} \leq \frac{\epsilon}{4K},
\end{align*}
where the last inequality holds due to $x A^{1-x^2}$ is monotonically decreasing w.r.t. $x$ in $[1, \infty)$ when $A \geq 2$. Therefore, we know that
\begin{align}
    (\text{ii}) = 4\sum_{k>A-K} \ind\{\Delta_k > A\epsilon/(4K)\}\exp\bigg(-\frac{n\Delta_k^2}{4A}\bigg)  \Delta_k \leq \epsilon. \label{eq:multi-armed-upper-x2}
\end{align}
Therefore, combining~\eqref{eq:multi-armed-upper-decomposition},~\eqref{eq:multi-armed-upper-x1} and~\eqref{eq:multi-armed-upper-x2}, we obtain that $\subopt(\pi) \leq 2\epsilon$ given 
\begin{align*}
    n \geq 8A\log\bigg(\frac{8A}{\epsilon}\bigg) + \frac{64K^2\log A}{A\epsilon^2} = \tilde{O}\bigg(\frac{K^2}{A\epsilon^2}\bigg).
\end{align*}
Finally, we express this guarantee in terms of $n$. Given $n$, we set $\epsilon = 8K\sqrt{\log A/(An)}$, so that $n = 64K^2\log A/(A\epsilon^2)$. Moreover, since $K \geq 1$ and $\log A \geq 1$ for $A \geq 3$, we have $8A/\epsilon = (A/K)\sqrt{An/\log A} \leq A^2n$, and thus $n \geq 8A\log(A^2n) \geq 8A\log(8A/\epsilon)$ by the assumption of Theorem~\ref{thm:multi-armed-upper}. Therefore, the argument above applies (the two requirements on $n$ are used separately in~\eqref{eq:multi-armed-upper-x1} and~\eqref{eq:multi-armed-upper-x2}), which gives $\EE[\subopt(\pihat)] \leq 2\epsilon = 16K\sqrt{\log A/(An)}$ and finishes the proof.
\end{proof}

\section{Auxiliary Lemmas}

\begin{lemma}[Lemma A.1, \citealt{xie2021policy}]\label{lem:binary-concentration}
Suppose $N \sim \mathsf{Bin}(n,p)$ where $n \geq 1$ and $p >0$, then
\begin{align*}
    \PP\Big[N \geq \frac{np}{2}\Big] \geq 1 - \exp(-np/8).
\end{align*}
Equivalently, with probability at least $1 - \delta$,
\begin{align*}
    \frac{p}{N \vee 1} \leq \frac{8\log(1/\delta)}{n}.
\end{align*}
\end{lemma}

\begin{lemma}[Azuma-Hoeffding's inequality, \citealt{azuma1967weighted}]\label{lem:azuma-hoeffding}
Let $Z_0, Z_1, \cdots, Z_n$ be a martingale sequence of random variables such that for all $i$, there exists a constant $c_i$ such that $|Z_i - Z_{i-1}| < c_i$, then
\begin{align*}
    \PP[Z_n - Z_0 \geq t] \leq \exp\bigg(-\frac{t^2}{2\sum_{i=1}^n c_i^2}  \bigg).
\end{align*}
In particular, if $\sup_{i} c_i \leq M$, then $\forall \delta \in (0, 1)$, the following inequality holds with probability at least $1-\delta$:
\begin{align*}
    Z_n - Z_0 \leq M\sqrt{2n\log(1/\delta)}.
\end{align*}
Remarkably, these inequalities also holds for martingale differences with sub-gaussian tails and we refer to~\citet{shamir2011variant} for a detailed derivation.
\end{lemma}

\begin{lemma}[Freedman's inequality, \citealt{freedman1975tail}]\label{lem:freedman}
Let $Z_0, Z_1, \cdots, Z_n$ be a martingale sequence of random variables such that for all $i$, $|Z_i - Z_{i-1}| < M$ and $\EE[(Z_i - Z_{i-1})^2 | Z_0, \ldots, Z_{i-1}] \leq \sigma_i^2$, then
\begin{align*}
    \PP[Z_n - Z_0 \geq t] \leq \exp\bigg(-\frac{t^2}{2\sum_{i=1}^n \sigma_i^2 + 2Mt/3}  \bigg).
\end{align*}
\end{lemma}

\begin{lemma}[Hoeffding's inequality]
\label{lem:Hoeffding}
Suppose that $\{X_i,i=1,\ldots,n\}$ are independent. Each $X_i$ has mean $\mu_i$ and sub-Gaussian parameter $\sigma_i$. Then, for all $t \ge 0$, we have
\begin{align*}
    \PP\bigg[\sum_{i=1}^n (X_i-\mu_i) \ge t\bigg] \le \exp \bigg\{\frac{-t^2}{2 \sum_{i=1}^n \sigma_i^2}\bigg\}.
\end{align*}
    
\end{lemma}
\begin{lemma}[Chernoff Bounds]
\label{lem:Chernoff}
    Assume that $\{X_i,i=1,\ldots,n\}$ are i.i.d random variables. Moreover, $\EE[X_1]=\mu$, $X_i \in [0,1]$. Then, for all $\epsilon > 0$, 
    \begin{align*}
        \PP\bigg[\frac{\sum_{i=1}^n X_i}{n} \ge (1+\epsilon)\mu \bigg] &\le \exp \bigg[\frac{-n\mu \epsilon^2}{2+\epsilon}\bigg],\\
        \PP\bigg[\frac{\sum_{i=1}^n X_i}{n} \le (1-\epsilon) \mu \bigg] & \le \exp \bigg[\frac{-n\mu \epsilon^2}{2}\bigg].
    \end{align*}
\end{lemma}

\Cref{lem:fano} is a standard reduction to Fano's inequality \citep{lecam1973convergence,yu1997assouad,polyanskiy2025information}. See, e.g., \citet[Section~3]{chen2024assouad} for a general proof.

\begin{lemma}[Fano's inequality]\label{lem:fano}
Fix any $\cR \coloneq \{r_1, \cdots, r_M\}$ with $M \geq 2$ and policy class $\Pi$, and let $L : \Pi \times \cR \to \RR_{+}$ be some loss function. Suppose there exists a constant $c > 0$ such that the following condition holds:
\begin{align*}
    \min_{i \neq j} \min_{\pi \in \Pi} L(\pi, r_i) + L(\pi, r_j) \geq c.
\end{align*}
Then we have
\begin{align*}
    \inf_{\pi \in \Pi} \sup_{r \in \cR} \EE_{\cD \sim P_r}  L(\pi, r) \geq \frac{c}{2} \bigg(1 - \frac{\max_{i \neq j}\KL(P_{r_i} \| P_{r_j}) + \log 2}{\log M} \bigg),
\end{align*}
where the trajectory distribution of $\pi$ interacting with instance $r$ is denoted by $P_r$.
\end{lemma}

The following \Cref{lem:max-signals} is due to \citet{gilbert1952comparison,varshamov1957estimate}, which is a classical result in coding theory. We refer the readers to Theorem~4.2.1 in~\citet{guruswami2019essential} for more details.

\begin{lemma}\label{lem:max-signals}
Suppose $\Sigma$ is a set of characters with $|\Sigma| = q$ where $q \geq 5$ is a prime power and $N>0$ is some natural number. Then there exists a subset $\cV$ of $\Sigma^{N}$ such that (1) for any $v, v' \in \cV, v \neq v'$, one has $d_H(v, v') \geq N/2$ and (2) $\log_q|\cV| / N \geq 1 - H_q(1/2) \geq 1/8$, where $d_H$ is the Hamming distance (i.e., the number of different entries) and the entropy function $H$ is given by
\begin{align*}
    H_q(x) = x\frac{\log (q-1)}{\log q} - x\frac{\log x}{\log q} - (1-x) \frac{\log (1-x)}{\log q}.
\end{align*}
In addition, for the binary alphabet $q=2$ with relative distance $1/4$, there exists a subset $\cV$ of $\{-1, 1\}^N$ such that (1) $|\cV| \geq \exp(N/8)$ and (2) for any
$v, v' \in \cV, v \neq v'$, one has $\|v - v'\|_1 \geq N/2$, i.e., $d_H(v,v') \geq N/4$.
\end{lemma}

Recall that a regular bipartite graph is a bipartite graph in which all vertices have the same degree.
\begin{lemma}[{\citealt[Corollary~5.2]{bondy1979graph}}]\label{lem:perfect-matching}
    Every regular bipartite graph with non-empty edge set has a perfect matching.
\end{lemma}

We adapt the following folklore (see, e.g., \citealt[Eq.~(D.10)]{zhao2026towards}).
\begin{lemma}\label{lem:kl-subopt-decompose}
Let $\cA$ be a finite action set, $\eta >0$, and $r: \cA \to \RR$ be some reward function. Let $\piref \in \cP(\cA)$ be some reference policy and $\pi^* \in \cP(\cA)$ be the optimal policy under $r$, i.e., $\pi^*(a) \propto \piref(a)\exp(\eta r(a))$ for all $a \in \cA$. For any policy $\pi$, the suboptimality gap between $\pi$ and $\pi^*$ under the KL-regularized objective is given by $\subopt(\pi, \pi^*) = \eta^{-1} \kl{\pi}{\pi^*}$.
\end{lemma}


\bibliographystyle{ims}
\bibliography{ref}

\end{document}